\documentclass{article}

\usepackage{microtype}
\usepackage{graphicx}
\usepackage{subfigure}
\usepackage{booktabs} % for professional tables

\usepackage{hyperref}

\usepackage[accepted]{icml2019}

\icmltitlerunning{Bayesian Deconditional Kernel Mean Embeddings}

\usepackage{amsmath}
\usepackage{amssymb}
\usepackage{amsthm}
\usepackage{mathbbol}
\usepackage{mathtools}
\usepackage{mathrsfs}
\usepackage{vector}
\usepackage{cleveref}
\usepackage{bm}
\usepackage[dvipsnames]{xcolor}
\newtheorem{theorem}{Theorem}[section]

\newtheorem{lemma}[theorem]{Lemma}

\usepackage{multirow}
\usepackage{algorithm}
\usepackage{algorithmic}
\usepackage{placeins}
\usepackage{tikz}
\usetikzlibrary{bayesnet, shapes, snakes}

\newcommand{\argmax}{\operatornamewithlimits{argmax}}
\newcommand{\argmin}{\operatornamewithlimits{argmin}}

\numberwithin{equation}{section}
\numberwithin{table}{section}
\numberwithin{algorithm}{section}

\theoremstyle{definition}
\newtheorem{definition}{Definition}[section]

\theoremstyle{remark}

\usepackage[colorinlistoftodos,prependcaption]{todonotes}
\usepackage[most]{tcolorbox}

\usepackage{glossaries}
\makeatletter
\let\oldmakefirstuc\makefirstuc
\renewcommand*{\makefirstuc}[1]{%
  \def\gls@add@space{}%
  \mfu@capitalisewords#1 \@nil\mfu@endcap
}
\def\mfu@capitalisewords#1 #2\mfu@endcap{%
  \def\mfu@cap@first{#1}%
  \def\mfu@cap@second{#2}%
  \gls@add@space
  \oldmakefirstuc{#1}%
  \def\gls@add@space{ }%
  \ifx\mfu@cap@second\@nnil
    \let\next@mfu@cap\mfu@noop
  \else
    \let\next@mfu@cap\mfu@capitalisewords
  \fi
  \next@mfu@cap#2\mfu@endcap
}
\makeatother

\newacronym{KME}{KME}{kernel mean embedding}
\newacronym{CME}{CME}{conditional mean embedding}
\newacronym{DME}{DME}{deconditional mean embedding}
\newacronym{CMO}{CMO}{conditional mean operator}
\newacronym{DMO}{DMO}{deconditional mean operator}
\newacronym{RKHS}{RKHS}{reproducing kernel Hilbert space}
\newacronym{HS}{HS}{Hilbert-Schmidt}
\newacronym{TTR}{TTR}{task transformed regression}
\newacronym{LRR}{LRR}{linear ridge regression}
\newacronym{KRR}{KRR}{kernel ridge regression}
\newacronym{BLR}{BLR}{Bayesian linear regression}
\newacronym{BKR}{BKR}{Bayesian kernel regression}
\newacronym{TLRR}{TLRR}{transformed LRR}
\newacronym{TKRR}{TKRR}{transformed KRR}
\newacronym{TBLR}{TBLR}{transformed BLR}
\newacronym{TBKR}{TBKR}{transformed BKR}
\newacronym{TTLRR}{TTLRR}{task transformed LRR}
\newacronym{TTKRR}{TTKRR}{task transformed KRR}
\newacronym{TTBLR}{TTBLR}{task transformed BLR}
\newacronym{TTBKR}{TTBKR}{task transformed BKR}
\newacronym{KBR}{KBR}{kernel Bayes' rule}
\newacronym{VI}{VI}{variational inference}
\newacronym{MCMC}{MCMC}{Markov chain Monte Carlo}
\newacronym{LFI}{LFI}{likelihood-free inference}
\newacronym{ABC}{ABC}{approximate Bayesian computation}
\newacronym{GP}{GP}{Gaussian process}
\newacronym{TGP}{TGP}{transformed Gaussian process}
\newacronym{TTGP}{TTGP}{task transformed Gaussian process}
\newacronym{GPR}{GPR}{Gaussian process regression}
\newacronym{TGPR}{TGPR}{transformed Gaussian process regression}
\newacronym{TTGPR}{TTGPR}{task transformed Gaussian process regression}
\newacronym{MAP}{MAP}{maximum a posteriori}
\newacronym{KDE}{KDE}{kernel density estimation}
\newacronym{REJ-ABC}{REJ-ABC}{rejection ABC}
\newacronym{MCMC-ABC}{MCMC-ABC}{Markov chain Monte carlo ABC}
\newacronym{SMC-ABC}{SMC-ABC}{sequential Monte carlo ABC}
\newacronym{LFVI}{LFVI}{likelihood-free variational inference}
\newacronym{MDN}{MDN}{mixture density network}
\newacronym{K-ABC}{K-ABC}{kernel ABC}
\newacronym{K2-ABC}{K2-ABC}{double kernel ABC}
\newacronym{GPS-ABC}{GPS-ABC}{Gaussian process surrogate ABC}
\newacronym{KELFI}{KELFI}{kernel embedding likelihood-free inference}

\definecolor{petroil2} {RGB} {36, 165, 175}
\definecolor{gold2} {RGB} {255, 130, 0}
\definecolor{aqua2} {RGB} {0, 180, 200}

\newcommand{\Width}{2}
\newcommand{\Height}{2}
\newcommand{\Depth}{2}
\newcommand{\midlinewidth}{0.8}

\begin{document}

\twocolumn[
\icmltitle{Bayesian Deconditional Kernel Mean Embeddings}



\icmlsetsymbol{equal}{*}

\begin{icmlauthorlist}
\icmlauthor{Kelvin Hsu}{usyd,csiro}
\icmlauthor{Fabio Ramos}{usyd,nvidia}
\end{icmlauthorlist}

\icmlaffiliation{usyd}{University of Sydney}
\icmlaffiliation{csiro}{CSIRO, Sydney}
\icmlaffiliation{nvidia}{NVIDIA, Seattle}

\icmlcorrespondingauthor{Kelvin Hsu}{Kelvin.Hsu@sydney.edu.au}

\icmlkeywords{Machine Learning, ICML}

\vskip 0.3in
]



\printAffiliationsAndNotice{}  % leave blank if no need to mention equal contribution

\begin{abstract}
	Conditional kernel mean embeddings form an attractive nonparametric framework for representing conditional means of functions, describing the observation processes for many complex models. However, the recovery of the original underlying function of interest whose conditional mean was observed is a challenging inference task. We formalize \textit{deconditional kernel mean embeddings} as a solution to this inverse problem, and show that it can be naturally viewed as a nonparametric Bayes' rule. Critically, we introduce the notion of \textit{task transformed Gaussian processes} and establish deconditional kernel means as their posterior predictive mean. This connection provides Bayesian interpretations and uncertainty estimates for deconditional kernel mean embeddings, explains their regularization hyperparameters, and reveals a marginal likelihood for kernel hyperparameter learning. These revelations further enable practical applications such as likelihood-free inference and learning sparse representations for big data.
\end{abstract}

\section{Introduction}
\label{sec:introduction}
	
	Observations of complex phenomena often lead to likelihoods that are described by a conditional mean. A widely applicable setting where this occurs is collecting observations under uncertain inputs, where the task is to learn a function $f : \mathcal{X} \to \mathbb{R}$ to model a real-valued response $z$ as a function of inputs $x \in \mathcal{X}$ without being able to query or measure $x$ directly to observe this phenomenon. Instead, another measured input $y \in \mathcal{Y}$ relates to $x$ through $p(x | y)$. Consequently, given $y$, the response $Z$ has mean $g(y) := \mathbb{E}[f(X) | Y = y]$, where $g$ is called the conditional mean of $f$. Furthermore, $p(x | y)$ is often only available as sample pairs $\{x_{i}, y_{i}\}_{i = 1}^{n}$, from simulations, algorithms, or separate experiments, making recovery of latent functions $f$ from conditional means $g$ a challenging inference task.
	
	Our first contribution begins with formulating \glspl{DME} as solutions to this inference problem by building upon the framework of \glspl{CME} \citep{song2013kernel}. We show that the \gls{DME} can be established as a nonparametric Bayes' rule in the \gls{RKHS} and used for likelihood-free Bayesian inference. In contrast to \gls{KBR} \citep{fukumizu2013kernel} which uses third order tensors that can result in vanishing priors, \glspl{DME} use second order tensors and avoids this problem.
	
	Together with \glspl{CME} and \gls{KBR}, \glspl{DME} form a critical part of the \gls{KME} \citep{muandet2017kernel} framework, where probabilistic rules can be represented nonparametrically as operators that are linear in the \gls{RKHS}. This greatly simplifies probabilistic inference without requiring parametrized distributions and compromising flexibility.
	
	Despite this connection, there are elements unique to the \gls{KME} framework that cannot be interpreted or solved via the parallel between probability rules and \gls{RKHS} mean operations. Similar to empirical forms for \gls{KBR} and \glspl{CME}, empirical \glspl{DME} are obtained by replacing expectations in its constituent operators with their empirical means, and introduce regularization for operator inverses to relax \gls{RKHS} assumptions, instead of as the optimal solution to a particular loss. Setting regularization hyperparameters is difficult in practice without an appropriate loss for the inference task. Furthermore, similar to \gls{KBR}, the nonparametric Bayes' rule provided by \glspl{DME} is a statement between observed (or simulated) variables and not on latent functions or quantities. Consequently, uncertainty estimation in inference of latent functions $f$ still require a separate Bayesian formulation.
	
	Our second contribution establishes a Bayesian view of \glspl{DME} as posterior predictive means of the \gls{TTGP}, a novel nonparametric Bayesian model that recover latent relationships between variables without observing them jointly. \Glspl{TTGP} are so named because we show that they are a type of transformed Gaussian process \citep{murray2005transformations} where the transformations and noise covariances are learned, by transforming one \gls{GP} task to another, rather than designed from expert knowledge. We use this connection to derive posterior and predictive uncertainty estimates for \glspl{DME} and explain their regularization hyperparameters as a function of noise variance. Finally, we derive marginal likelihoods and their scalable computational forms to learn \gls{DME} hyperparameters, which can also be applied to learn inducing points for sparse representations as a special case. All proofs are in the supplementary material.

\vspace{-0.5em}
\section{Kernel Mean Embeddings}
\label{sec:kernel_mean_embedding}

	We begin with an overview of the \gls{KME} framework from which \glspl{DME} are built upon. \glspl{KME} are an arsenal of techniques concerned with representations and transformations of function expectations under highly flexible distributions. They consider functions that lie within \glspl{RKHS} $\mathcal{H}_{k}$ and $\mathcal{H}_{\ell}$, formed by positive definite kernels $k : \mathcal{X} \times \mathcal{X} \to \mathbb{R}$ and $\ell : \mathcal{Y} \times \mathcal{Y} \to \mathbb{R}$. The \glspl{RKHS} $\mathcal{H}_{k}$ and $\mathcal{H}_{\ell}$ are the closure span of the features $\phi(x) = k(x, \cdot)$ and $\psi(y) = \ell(y, \cdot)$ across $x \in \mathcal{X}$ and $y \in \mathcal{Y}$ respectively, endowed with the inner products $\langle \cdot, \cdot \rangle_{k} \equiv \langle \cdot, \cdot \rangle_{\mathcal{H}_{k}}$ and $\langle \cdot, \cdot \rangle_{\ell} \equiv \langle \cdot, \cdot \rangle_{\mathcal{H}_{\ell}}$.
	
	The key object is the mean embedding of a distribution $\mu_{X} := \mathbb{E}[k(X, \cdot)] \in \mathcal{H}_{k}$. They encode function expectations in the sense that $\mathbb{E}[f(X)] = \langle \mu_{X}, f \rangle_{k}$, due to the reproducing property that $\langle k(x, \cdot), f \rangle_{k} = f(x)$ for all $f \in \mathcal{H}_{k}$.
	
	Higher ordered mean embeddings are vital components of the framework. Specifically, second order mean embeddings such as $C_{YY} := \mathbb{E}[\ell(Y, \cdot) \otimes \ell(Y, \cdot)] \in \mathcal{H}_{\ell} \otimes \mathcal{H}_{\ell}$ and $C_{XY} := \mathbb{E}[k(X, \cdot) \otimes \ell(Y, \cdot)] \in \mathcal{H}_{k} \otimes \mathcal{H}_{\ell}$ can be identified as cross-covariance operators $C_{YY} : \mathcal{H}_{\ell} \to \mathcal{H}_{\ell}$ and $C_{XY} : \mathcal{H}_{\ell} \to \mathcal{H}_{k}$ that serve as building blocks of \glspl{CME} and \glspl{DME}.
	
	In practical scenarios where only \textit{iid} samples $\{x_{i}, y_{i}\}_{i = 1}^{n}$ that are realizations of $(X_{i}, Y_{i}) \sim \mathbb{P}_{XY}$ for $i \in \{1, \dots, n\}$ are available, the \gls{KME} framework becomes attractive for nonparametric inference because core objects only require expectations under distributions. Consequently, they can be estimated via empirical means as $\hat{\mu}_{X} := \frac{1}{n} \sum_{i = 1}^{n} k(x_{i}, \cdot)$, $\hat{C}_{YY} := \frac{1}{n} \sum_{i = 1}^{n} \ell(y_{i}, \cdot) \otimes \ell(y_{i}, \cdot)$, and $\hat{C}_{XY} := \frac{1}{n} \sum_{i = 1}^{n} k(x_{i}, \cdot) \otimes \ell(y_{i}, \cdot)$ \citep{muandet2017kernel}.
	
	For feature matrices, we stack features by columns $\Phi := \begin{bmatrix} \phi(x_{1}) & \cdots & \phi(x_{n}) \end{bmatrix}$ and $\Psi := \begin{bmatrix} \psi(y_{1}) & \cdots & \psi(y_{n}) \end{bmatrix}$. We write gram matrices as $K := \Phi^{T} \Phi$ and $L := \Psi^{T} \Psi$, where the $(i, j)$-th element of $A^{T} B$ is the inner product of the $i$-th column of $A$ with the $j$-th column of $B$. That is, $K_{ij} = \phi(x_{i})^{T} \phi(x_{j})$ and $L_{ij} = \psi(y_{i})^{T} \psi(y_{j})$. When columns are elements of \glspl{RKHS} such as when $\phi(x) = k(x, \cdot)$ in $\Phi$ and $\psi(y) = \ell(y, \cdot)$ in $\Psi$, the notation $(\cdot)^{T}(\cdot)$ is a shorthand for the corresponding \gls{RKHS} inner product $\langle \cdot, \cdot \rangle_{\mathcal{H}}$ when it is clear from context what $\mathcal{H}$ is. For example, $f^{T} h$ is shorthand for $\langle f, h \rangle_{k}$ if $f, h \in \mathcal{H}_{k}$. Another common usage is $\Phi^{T} f = \{\phi(x_{i})^{T} f\}_{i = 1}^{n} = \{k(x_{i}, \cdot)^{T} f\}_{i = 1}^{n} = \{\langle k(x_{i}, \cdot), f \rangle_{k}\}_{i = 1}^{n} = \{f(x_{i})\}_{i = 1}^{n} =: \bvec{f}$. For summing outer products, we write $\hat{C}_{YY} = \frac{1}{n} \Psi \Psi^{T}$ and $\hat{C}_{XY} = \frac{1}{n} \Phi \Psi^{T}$. Note that we use non-bold letters for single points $x$ and $y$, even though they are often multivariate in practice.
	
\section{Conditional Kernel Mean Embeddings}
\label{sec:conditional_mean_embedding}

	We now present \glspl{CME} in a fashion that focuses on their operator properties. By reviewing \glspl{CME} this way, parallels and contrast with \glspl{DME} in the subsequent \cref{sec:deconditional_mean_embedding} become more apparent. Importantly, instead of defining \glspl{CME} via an explicit form, we begin by forming problem statements.
		\begin{definition}[Conditional Mean Problem Statement]
		\label{def:cmp}
			Given a function $f: \mathcal{X} \to \mathbb{R}$, infer the function $g: \mathcal{Y} \to \mathbb{R}$ such that
					$g(y) = \mathbb{E}[f(X) | Y = y] \equiv \mathbb{E}_{X|Y}[f](y)$. We call $g$ the \textit{conditional mean} of $f$ with respect to $\mathbb{P}_{X|Y}$ and write the shorthand $g = \mathbb{E}_{X|Y}[f] = \mathbb{E}[f(X) | Y = \cdot]$.	
		\end{definition}
	This naturally leads to the notion of operators that map functions $f$ to their conditional means $g = \mathbb{E}[f(X) | Y = \cdot]$.
		\begin{definition}[Conditional Mean Operators]
		\label{def:cmo}
			The \gls{CMO} $C_{X|Y} : \mathcal{H}_{\ell} \to \mathcal{H}_{k}$ corresponding to $\mathbb{P}_{X|Y}$ is the operator that satisfies
			\begin{equation}
				(C_{X|Y})^{T} f = \mathbb{E}[f(X) | Y = \cdot], \quad \forall f \in \mathcal{H}_{k},
			\label{eq:cmo_def}
			\end{equation}
			where $(C_{X|Y})^{T} : \mathcal{H}_{k} \to \mathcal{H}_{\ell}$ denotes the adjoint of $C_{X|Y}$.
		\end{definition}
	Depending on the nature of $\ell$, unique solutions exist.
			\begin{theorem}[\citep{fukumizu2004dimensionality}]
				\label{thm:cmo}
				Assume that $\ell(y, \cdot) \in \mathrm{image}(C_{YY})$ for all $y \in \mathcal{Y}$. The \acrfull{CMO} $C_{X|Y}$ is unique and given by
				\begin{equation}
					C_{X|Y} = C_{XY} C_{YY}^{-1}.
				\label{eq:cmo}
				\end{equation}
			\end{theorem}
	\vspace{-1em}
	The assumption that $\ell(y, \cdot) \in \mathrm{image}(C_{YY})$ for all $y \in \mathcal{Y}$ is commonly relaxed by introducing a regularization hyperparameter $\lambda > 0$ to the inverse, so that the \gls{CMO} is replaced with $C_{XY} (C_{YY} + \lambda I)^{-1}$ \citep{song2013kernel}.
	
	Contrary to \cref{def:cmo}, it is more common in the literature to define the \gls{CMO} as the operator $C_{X|Y}$ that satisfies
	\vspace{-1em}
			\begin{equation}
				C_{X|Y} \ell(y, \cdot) = \mathbb{E}[k(X, \cdot) | Y = y], \quad \forall y \in \mathcal{Y},
			\label{eq:cmo_property}
			\end{equation}
	while \eqref{eq:cmo_def} is taken as an immediate property of \glspl{CMO} \citep{fukumizu2004dimensionality}.
	However, due to \cref{thm:cmo_definition_equivalence}, we instead take \cref{def:cmo} as the definition of \glspl{CMO}, emphasizing \glspl{CMO} as solutions to the conditional mean problem, and treat \eqref{eq:cmo_property} as an immediate property.
			\begin{lemma}
			\label{thm:cmo_definition_equivalence}
				Statements \eqref{eq:cmo_def} and \eqref{eq:cmo_property} are equivalent.
			\end{lemma}
	The \gls{CME} of $\mathbb{P}_{X|Y = y}$ is $\mu_{X|Y=y} := C_{X|Y} \ell(y, \cdot)$, equivalent to querying the \gls{CMO} at a particular input $y$. 
	
	Consequently, 
		$\langle \mu_{X|Y=y}, f \rangle_{k}
		= \langle C_{X|Y} \ell(y, \cdot), f \rangle_{k} 
		= \langle \ell(y, \cdot), (C_{X|Y})^{T} f \rangle_{\ell}
		= \langle \ell(y, \cdot), \mathbb{E}_{X|Y}[f] \rangle_{\ell}
		= \mathbb{E}_{X|Y}[f](y)$.
	
	Motivated by \cref{thm:cmo}, empirical \glspl{CMO} and \glspl{CME} are defined by estimating their constituents by empirical means.
		\begin{definition}[Empirical Conditional Mean Operator]
		\label{def:emp_cmo}
			The empirical \gls{CMO} is $\hat{C}_{X|Y} := \hat{C}_{XY} (\hat{C}_{YY} + \lambda I)^{-1}$, $\lambda > 0$.
		\end{definition}
		\begin{theorem}[\citep{song2009hilbert}]
		\label{thm:emp_cmo}
			The nonparametric form for $\hat{C}_{X|Y}$ is
			\vspace{-1.2em}
			\begin{equation}
				\hat{C}_{X|Y} = \Phi (L + n \lambda I)^{-1} \Psi^{T}.
			\label{eq:emp_cmo}
			\end{equation}
		\end{theorem}
	The empirical \gls{CME} is then $\hat{\mu}_{X|Y=y} := \hat{C}_{X|Y} \ell(y, \cdot)$.

	Consequently, with $\bm{\ell}(y) := \{\ell(y_{i}, y)\}_{i = 1}^{n}$, an estimate for $\mathbb{E}_{X|Y}[f](y)$ is
			$\langle f, \hat{\mu}_{X|Y=y} \rangle_{k}
					= \langle f, \hat{C}_{X|Y} \ell(y, \cdot) \rangle_{k}
					= f^{T} \Phi (L + n \lambda I)^{-1} \Psi^{T} \ell(y, \cdot)
					= \bvec{f}^{T} (L + n \lambda I)^{-1} \bm{\ell}(y)$.
		
	Critically, while empirical \glspl{CMO} \eqref{eq:emp_cmo} are estimated from joint samples from the joint distribution $\mathbb{P}_{XY}$, they only encode the conditional distribution $\mathbb{P}_{X|Y}$. This means that the empirical \glspl{CMO} will encode the same conditional distribution even if the joint distribution $\mathbb{P}_{XY}$ changes but the conditional distribution $\mathbb{P}_{X|Y}$ stays the same. That is, the empirical \gls{CMO} built from joint samples of $p(x, y) = p(x | y) p(y)$ and the empirical \gls{CMO} built from joint samples of $q(x, y) := p(x | y) q(y)$ will encode the same conditional distribution $p(x | y)$ and converge to the same \gls{CMO}.

\section{Deconditional Kernel Mean Embeddings}
\label{sec:deconditional_mean_embedding}

	We now present a novel class of \glspl{KME} referred to as \acrfullpl{DME}. They are natural counterparts to \glspl{CME}. The presentation of definitions and theorems in this section is mainly parallel to \cref{sec:conditional_mean_embedding}. We define the \textit{deconditional mean} problem as the task of recovering latent functions from their conditional means.
		\begin{definition}[Deconditional Mean Problem Statement]
		\label{def:dmp}
			Given a function $g: \mathcal{Y} \to \mathbb{R}$, infer a function $f: \mathcal{X} \to \mathbb{R}$ such that $g(y) = \mathbb{E}[f(X) | Y = y]$. We call $f$ a \textit{deconditional mean} of $g$ with respect to $\mathbb{P}_{X|Y}$ and write the shorthand $f = \mathbb{E}^{\dagger}_{X|Y}[g]$. 
		\end{definition}
	The deconditional mean of a function $g$ infers the function $f$ whose conditional mean would be $g$ with respect to $\mathbb{P}_{X|Y}$. The corresponding operator that encodes this transformation is the \gls{DMO}.
		\begin{definition}[Deconditional Mean Operators]
		\label{def:dmo}
			The \acrfull{DMO} $C_{X|Y}' : \mathcal{H}_{k} \to \mathcal{H}_{\ell}$ corresponding to $\mathbb{P}_{X|Y}$ is the operator that satisfies
			\begin{equation}
				(C_{X|Y}')^{T} \mathbb{E}[f(X) | Y = \cdot] = f, \quad \forall f \in \mathcal{H}_{k}.
				\label{eq:dmo_def}
			\end{equation}
		\end{definition}
	\vspace{-0.8em}
	Depending on the nature of $\ell$ and $k$, unique solutions exist.
		\begin{theorem}
		\label{thm:dmo}
			Assume that $\ell(y, \cdot) \in \mathrm{image}(C_{YY})$ for all $y \in \mathcal{Y}$ and $k(x, \cdot) \in \mathrm{image}(C_{X|Y} C_{YY} (C_{X|Y})^{T})$ for all $x \in \mathcal{X}$. The \acrfull{DMO} $C_{X|Y}'$ is unique and given by 
			\begin{equation}
				C_{X|Y}' = (C_{X|Y} C_{YY})^{T} (C_{X|Y} C_{YY} (C_{X|Y})^{T})^{-1}.
			\label{eq:dmo}
			\end{equation}
		\end{theorem}
	\vspace{-0.8em}
	Similar to the case with \glspl{CMO} \citep{song2013kernel}, the assumption that $k(x, \cdot) \in \mathrm{image}(C_{X|Y} C_{YY} (C_{X|Y})^{T})$ for all $x \in \mathcal{X}$ can be relaxed by introducing a regularization hyperparameter $\epsilon > 0$ to the inverse, so that the \gls{DMO} is replaced with $(C_{X|Y} C_{YY})^{T} (C_{X|Y} C_{YY} (C_{X|Y})^{T} + \epsilon I)^{-1}$.
	
	Since \glspl{DMO} invert the results of \glspl{CMO}, they can also be understood as pseudo-inverses of \glspl{CMO}.
		\begin{theorem}
		\label{thm:dmo_is_pseudo_inv_of_cmo}
			If the assumptions in \cref{thm:dmo} hold and further $((C_{X|Y})^{T} C_{X|Y})^{-1}$ exists such that the pseudo-inverse $C^{\dagger}_{X|Y} := ((C_{X|Y})^{T} C_{X|Y})^{-1} (C_{X|Y})^{T}$ is defined, then \glspl{DMO} are pseudo-inverses of \glspl{CMO} $C_{X|Y}' = C_{X|Y}^{\dagger}$.
		\end{theorem}
	The \gls{DME} of $\mathbb{P}_{X=x|Y}$ is $\mu_{X=x|Y}' := C_{X|Y}' k(x, \cdot) \in \mathcal{H}_{\ell}$, equivalent to querying the \gls{DMO} at a particular input $x$.

	Consequently,
		$\langle \mu_{X=x|Y}' , g \rangle_{\ell} = \langle C_{X|Y}' k(x, \cdot), g \rangle_{\ell}
		= \langle k(x, \cdot), (C_{X|Y}')^{T} g \rangle_{k}
		= \langle k(x, \cdot), f \rangle_{k}
		= f(x)$.
		
	The form in \eqref{eq:dmo} makes it evident that a \gls{DMO} can be fully specified once $C_{X|Y}$ and $C_{YY}$, encoding the measures $\mathbb{P}_{X|Y}$ and $\mathbb{P}_{Y}$ respectively, are known. If densities exist, we write them as $p_{X|Y} \equiv p_{X|Y}(\cdot|\cdot)$ and $p_{Y} \equiv p_{Y}(\cdot)$, and drop the subscripts in density evaluations as $p(x|y)$ and $p(y)$ whenever the context is clear. Note that $\mathbb{P}_{X=x|Y}$ corresponds to $p_{X|Y}(x | \cdot)$ which is evaluated at $x$ and now a function of $y$. This is in contrast with $\mathbb{P}_{X|Y=y}$ corresponding to $p_{X|Y}(\cdot | y)$ evaluated at $y$ and now a function of $x$.
		
	Consider the case where $X$ and $Y$ play the roles of observed and unobserved (latent) variables respectively. The \gls{DMO} considers the conditional $p_{X|Y}$ and the marginal $p_{Y}$ encoded as $C_{X|Y}$ and $C_{YY}$ (\cref{thm:dmo}), and inverts the \gls{CMO} $C_{X|Y}$ (\cref{thm:dmo_is_pseudo_inv_of_cmo}) with the help of the encoded marginal $C_{YY}$. This is analogous to the Bayes' rule, where the posterior $p_{Y|X}(\cdot|x) = \frac{p_{X|Y}(x|\cdot) p_{Y}(\cdot)}{\int_{\mathcal{Y}} p_{X|Y}(x|y) p_{Y}(y) dy}$ is fully specified by the likelihood $p_{X|Y}$ and prior $p_{Y}$. We can then interpret \glspl{DME} as querying the rule at the observed quantity $x$ while leaving the rule as a function of $y$ for inference. Consequently, we also refer to $C_{X|Y}$ and $C_{YY}$ as the likelihood operator and the prior operator respectively. 

	The difference between the \gls{DMO} \eqref{eq:dmo} and \gls{CMO} \eqref{eq:cmo} equations is akin to writing $p_{Y|X}(\cdot | x)$ using Bayes' rule against using the conditional density rule. Compare the \gls{DMO} decomposition \eqref{eq:dmo} with the \gls{CMO} decomposition $C_{Y|X} = C_{YX} C_{XX}^{-1} = (C_{XY})^{T} C_{XX}^{-1}$ in the other direction by reversing the roles of $X$ and $Y$ in \eqref{eq:cmo}, which would correspond to the posterior $\mathbb{P}_{Y|X}$. The \gls{CMO} is composed of a \textit{joint} operator $C_{XY}$ and an \textit{evidence} operator $C_{XX}$ corresponding to the joint $\mathbb{P}_{XY}$ and evidence $\mathbb{P}_{X}$ distributions. Similarly, the \gls{DMO} is also composed of a \textit{joint} operator $C_{XY} = C_{X|Y} C_{YY} : \mathcal{H}_{\ell} \to \mathcal{H}_{k}$ and an \textit{evidence} operator $C_{XX}' := C_{X|Y} C_{YY} (C_{X|Y})^{T} : \mathcal{H}_{k} \to \mathcal{H}_{k}$, but both specified from the likelihood and prior operators. 
	
	Motivated by this, we propose to estimate the likelihood and prior operators using separate and independently drawn samples. The likelihood operator $C_{X|Y}$ is estimated as $\hat{C}_{X|Y}$ (\cref{def:emp_cmo}) using \textit{iid} samples $\{x_{i}, y_{i}\}_{i = 1}^{n}$, also denoted as $\bvec{x} := \{x_{i}\}_{i = 1}^{n}$ and $\bvec{y} := \{y_{i}\}_{i = 1}^{n}$. Note that as the likelihood operator is a \gls{CMO}, these joint samples can be from any joint distribution $\mathbb{Q}_{XY} \neq \mathbb{P}_{XY}$ as long as its conditional distribution is also $\mathbb{P}_{X|Y}$. The prior operator $C_{YY}$ is estimated as $\tilde{C}_{YY} := \frac{1}{m} \sum_{j = 1}^{m} \ell(\tilde{y}_{j}, \cdot) \otimes \ell(\tilde{y}_{j}, \cdot)$ using another set of \textit{iid} samples $\tilde{\bvec{y}} := \{\tilde{y}_{j}\}_{j = 1}^{m}$ from $\mathbb{P}_{Y}$.
		\begin{definition}[Empirical Deconditional Mean Operator]
		\label{def:emp_dmo}
			Let $\epsilon > 0$ be a regularization hyperparameter and define $\hat{C}_{X|Y}$ and $\tilde{C}_{YY}$ as above. The empirical \gls{DMO} is
			\begin{equation}
				\bar{C}_{X|Y}' := (\hat{C}_{X|Y} \tilde{C}_{YY})^{T} (\hat{C}_{X|Y} \tilde{C}_{YY} (\hat{C}_{X|Y})^{T} + \epsilon I)^{-1}.
			\label{eq:emp_dmo_def}
			\end{equation}
		\end{definition}
	\vspace{-0.8em}
	The accents notate the set of samples used for estimation. When both sets are used such as in the estimation of the \gls{DMO} $C_{X|Y}'$, we denote it with a bar such as $\bar{C}_{X|Y}'$.	
		\begin{theorem}
		\label{thm:emp_dmo}
			The nonparametric form for $\bar{C}_{X|Y}'$ is
			\vspace{-0.5em}
			\begin{equation}
				\bar{C}_{X|Y}' = \tilde{\Psi} \big[ A^{T} K A + m \epsilon I \big]^{-1} A^{T} \Phi^{T},
				\vspace{-0.5em}
			\label{eq:emp_dmo}
			\end{equation}
			where $A := (L + n \lambda I)^{-1} \tilde{L}$, $\tilde{L} := \Psi^{T} \tilde{\Psi}$, and $\tilde{\Psi} := \begin{bmatrix} \psi(\tilde{y}_{1}) & \cdots & \psi(\tilde{y}_{m}) \end{bmatrix}$.
		\end{theorem}
	The empirical \gls{DME} is then $\bar{\mu}_{X=x|Y}' := \bar{C}_{X|Y}' k(x, \cdot)$.
	
	Consequently, with $\bm{k}(x) := \{k(x_{i}, x)\}_{i = 1}^{n}$ and $\tilde{\bvec{g}} := \{g(\tilde{y}_{j})\}_{j = 1}^{m}$, an estimate for $\mathbb{E}^{\dagger}_{X|Y}[g](x)$ is $\langle g, \bar{\mu}_{X=x|Y}' \rangle_{\ell} = \tilde{\bvec{g}}^{T} \big[ A^{T} K A + m \epsilon I \big]^{-1} A^{T} \bm{k}(x)$. This motivates the following definitions, where the notation $\tilde{\bvec{g}}$ is replaced with $\tilde{\bvec{z}}$, to be interpreted as target observations of $g$ at $\tilde{\bvec{y}}$.
		\begin{definition}[Nonparametric \gls{DME} Estimator]
		\label{def:nonparametric_dme_estimator}
			The nonparametric \gls{DME} estimator, also called the kernel \gls{DME} estimator or the \gls{DME} estimator in function space view, is 
				$\bar{f}(x) = \bar{\bm{\alpha}}^{T} \bm{k}(x) = \sum_{i = 1}^{n} \bar{\alpha}_{i} k(x_{i}, x)$,
			where 
				$\bar{\bm{\alpha}} := A \big[ A^{T} K A + m \epsilon I \big]^{-1} \tilde{\bvec{z}}$ and $A := (L + n \lambda I)^{-1} \tilde{L}$.
			Equivalently,
				$\bar{f}(x) = \tilde{\bvec{z}}^{T} \big[ A^{T} K A + m \epsilon I \big]^{-1} A^{T} \bm{k}(x)$.
			An alternative form is
				$\bar{f}(x) = \tilde{\bvec{z}}^{T} A^{T} \big[ K A A^{T} + m \epsilon I \big]^{-1} \bm{k}(x)$.
		\end{definition}	
	When features $\phi(x) \in \mathbb{R}^{p}$ and $\psi(y) \in \mathbb{R}^{q}$ are finite dimensional, we define the parametric \gls{DME} estimator as follows by rewriting \cref{def:nonparametric_dme_estimator} using the Woodbury identity
		\begin{definition}[Parametric \gls{DME} Estimator]
		\label{def:parametric_dme_estimator}
			The parametric \gls{DME} estimator, also called the feature \gls{DME} estimator or the \gls{DME} estimator in weight space view, is
				$\bar{f}(x) = \bar{\bvec{w}}^{T} \phi(x)$,
			where 
				$\bar{\bvec{w}} = [\Phi A A^{T} \Phi^{T} + m \epsilon I]^{-1} \Phi A \tilde{\bvec{z}}$ and $A := \Psi^{T} (\Psi \Psi^{T} + n \lambda I)^{-1} \tilde{\Psi}$.
			Equivalently,
				$\bar{f}(x) = \tilde{\bvec{z}}^{T} A^{T} \Phi^{T} [\Phi A A^{T} \Phi^{T} + m \epsilon I]^{-1} \phi(x)$.
		\end{definition}
	In \cref{def:nonparametric_dme_estimator} (resp. \ref{def:parametric_dme_estimator}), computational complexity is dominated by inversions for $L + n \lambda I$ and $A^{T} K A + m \epsilon I$ (resp. $\Psi \Psi^{T} + n \lambda I$ and $\Phi A A^{T} \Phi^{T} + m \epsilon I$) at $O(n^{3})$ and $O(m^{3})$ (resp. $O(q^{3})$ and $O(p^{3})$). For the alternative form in \cref{def:nonparametric_dme_estimator}, both inversions are $O(n^{3})$, allowing for larger $m$ at $O(m)$ without compromising tractability.

\section{Task Transformed Gaussian Processes}
\label{sec:task_transformed_gaussian_process}

	\begin{figure}
	\centering
	\begin{tikzpicture}
	
		\node (LRR) at (0,0,0) {$\scriptstyle\mathsf{LRR}$};
		\node (KRR) at (0,\Width,0) {$\scriptstyle\mathsf{KRR}$};
		\node (BLR) at (\Depth,0,0){$\scriptstyle\mathsf{BLR}$};
		\node (TLRR) at (0,0,\Height){$\scriptstyle\mathsf{TLRR}$};
		\node (BKR) at (\Depth,\Width,0){$\scriptstyle\mathsf{BKR}$};
		\node (TKRR) at (0,\Width,\Height) {$\scriptstyle\mathsf{TKRR}$};
		\node (TBLR) at (\Depth,0,\Height){$\scriptstyle\mathsf{TBLR}$};
		\node (TBKR) at (\Depth,\Width,\Height){$\scriptstyle\mathsf{TBKR}$};

		\draw[->, orange!80!yellow,line width=\midlinewidth] (LRR) -- (TLRR);
		\draw[->, orange!80!yellow,line width=\midlinewidth] (KRR) -- (TKRR);
		\draw[->, orange!80!yellow,line width=\midlinewidth] (BLR) -- (TBLR);
		\draw[->, orange!80!yellow,line width=\midlinewidth] (BKR) -- (TBKR);
				
		\draw[->, cyan!80!blue,line width=\midlinewidth] (LRR) -- (KRR);
		\draw[->, cyan!80!blue,line width=\midlinewidth] (BLR) -- (BKR);
		
		\draw[->, green!80!blue,line width=\midlinewidth] (LRR) -- (BLR);
		\draw[->, green!80!blue,line width=\midlinewidth] (KRR) -- (BKR);
		\draw[->, green!80!blue,line width=\midlinewidth] (TLRR) -- (TBLR);
		\draw[->, green!80!blue,line width=\midlinewidth] (TKRR) -- (TBKR);

		\draw[->, cyan!80!blue,line width=\midlinewidth] (TLRR) -- (TKRR);
		\draw[->, cyan!80!blue,line width=\midlinewidth] (TBLR) -- (TBKR);
	
		\node[anchor=south,scale=0.75] 
		at ([xshift=2cm,yshift=1.5cm]current bounding box.south east)
		{
		\begin{tabular}{|cl|}
		\hline
		{\color{cyan!80!blue} $\rightarrow$} & Kernelize features (inputs) \\
		{\color{green!80!blue} $\rightarrow$} & Bayesian inference (latents) \\
		{\color{orange!80!yellow} $\rightarrow$} & Transform observations (outputs) \\
		\hline
		\end{tabular}
		};
	
	\end{tikzpicture}
	\vspace{-0.5em}
	\caption{Three dimensions of model extensions (T: Transformed, B: Bayesian, K/L: Kernel/Linear, (R)R: (Ridge) Regression). Kernel extensions (blue) specify feature spaces implicitly through a kernel. Bayesian extensions (green) introduce notions of uncertainty on latent quantities (weights or functions). Finally, transformed extensions (orange) capture indirect function observations.}
	\vspace{-1em}
	\label{fig:model-cube}
	\end{figure}
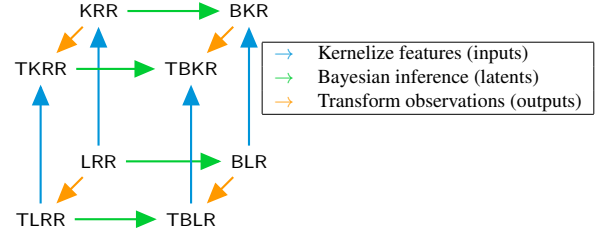
	
	\Glspl{DME} are constructed as solutions to the task of inferring deconditional means, which are often real-valued functions. Regression problems also address inference of real-valued functions from data. This raises curiosity towards whether \glspl{DME} can be formulated as solutions to a regression-like problem, and what insights this connection would provide.
		
	In this section, we formulate the task transformed regression problem to provide regression views of \glspl{DME}. To do this, we first briefly review transformed regression in \cref{sec:transformed_regression} before we present our contributions in \cref{sec:task_transformed_regression}.
		
	\subsection{Transformed Regression}
	\label{sec:transformed_regression}
	
		Standard regression models often assume a Gaussian full data likelihood $p(\bvec{z} | \bvec{f}) = \mathcal{N}(\bvec{z} ; \bvec{f}, \sigma^{2} I)$ with targets $\bvec{z} := \{z_{i}\}_{i = 1}^{n} \in \mathbb{R}^{n}$. In the generalized setting when observations of $\bvec{f}$ at $\bvec{x}$ are not available but observations of linear combinations thereof are, we can use $p(\tilde{\bvec{z}} | \bvec{f}) = \mathcal{N}(\tilde{\bvec{z}} ; M^{T} \bvec{f}, \Sigma)$ for some transformation $M \in \mathbb{R}^{n \times m}$ and noise covariance $\Sigma$, where $\tilde{\bvec{z}} := \{\tilde{z}_{j}\}_{j = 1}^{m} \in \mathbb{R}^{m}$ are the available observations. 
		
		\begin{figure}[!t]
		\begin{tikzpicture}[x=0.75cm,y=0.5cm]
		
			\node[obs]							(tilde_z_1) {$\tilde{z}_{1}$} ; %
			\node[latent, above=of tilde_z_1]	(tilde_g_1) {$\tilde{g}_{1}$} ; %
			\node[const, above=of tilde_g_1]	(tilde_y_1) {$\tilde{y}_{1}$} ; %
		
			\edge {tilde_y_1}  {tilde_g_1};
			\edge {tilde_g_1}  {tilde_z_1};
			
			\node[obs, right=of tilde_z_1]		(tilde_z_m) {$\tilde{z}_{m}$} ; %
			\node[latent, above=of tilde_z_m]	(tilde_g_m) {$\tilde{g}_{m}$} ; %
			\node[const, above=of tilde_g_m]	(tilde_y_m) {$\tilde{y}_{m}$} ; %
		
			\edge {tilde_y_m}  {tilde_g_m};
			\edge {tilde_g_m}  {tilde_z_m};
			\edge[dashed, -, line width=1mm] {tilde_g_1}  {tilde_g_m};
			
			\node[latent, right=of tilde_z_m]	(z_1) {$z_{1}$} ; %
			\node[latent, above=of z_1]			(g_1) {$g_{1}$} ; %
			\node[const, above=of g_1]			(y_1) {$y_{1}$} ; %
		
			\edge {y_1}  {g_1};
			\edge {g_1}  {z_1};
			\edge[-, line width=1mm] {tilde_g_m}  {g_1};
			
			\node[latent, right=of z_1]			(z_n) {$z_{n}$} ; %
			\node[latent, above=of z_n]			(g_n) {$g_{n}$} ; %
			\node[const, above=of g_n]			(y_n) {$y_{n}$} ; %
		
			\edge {y_n}  {g_n};
			\edge {g_n}  {z_n};
			\edge[dashed, -, line width=1mm] {g_1}  {g_n};
			
			\node[latent, below=of z_1] 		(f_1) {$f_{1}$} ; %
			\node[const, below=of f_1] 			(x_1) {$x_{1}^{{\color{white} \star}}$} ; %
			
			\edge {f_1} {z_1};
			\edge {x_1} {f_1};
			
			\node[latent, below=of z_n] 		(f_n) {$f_{n}$} ; %
			\node[const, below=of f_n] 			(x_n) {$x_{n}^{{\color{white} \star}}$} ; %
			
			\edge {f_n} {z_n};
			\edge {x_n} {f_n};	
			\edge[dashed, -, line width=1mm] {f_1}  {f_n};
			
			\node[latent, right=of z_n] 		(z_1_q) {$z_{1}^{\star}$} ; %
			\node[latent, below=of z_1_q] 		(f_1_q) {$f_{1}^{\star}$} ; %
			\node[const, below=of f_1_q] 		(x_1_q) {$x_{1}^{\star}$} ; %
			
			\edge {f_1_q} {z_1_q};
			\edge {x_1_q} {f_1_q};
			\edge[-, line width=1mm] {f_n}  {f_1_q};
			
			\node[latent, right=of z_1_q] 		(z_q_q) {$z_{q}^{\star}$} ; %
			\node[latent, below=of z_q_q] 		(f_q_q) {$f_{q}^{\star}$} ; %
			\node[const, below=of f_q_q] 		(x_q_q) {$x_{q}^{\star}$} ; %
			
			\edge {f_q_q} {z_q_q};
			\edge {x_q_q} {f_q_q};	
			\edge[dashed, -, line width=1mm] {f_1_q}  {f_q_q};	
		
		\end{tikzpicture}
		\vspace{-0.5em}
		\caption{Graphical model (chain graph) for \acrfull{TTGPR}. Circles are random variables. Shaded circles are observed random variables. Undirected edges indicate the \gls{GP} field, where all the random variables on the field are fully connected to each other \citep{rasmussen2006gaussian}. The goal is to infer $\bvec{f}^{\star}$ to predict $\bvec{z}^{\star}$ at $\bvec{x}^{\star}$, using only a \textit{task} or \textit{original} dataset $\{\tilde{y}_{j}, \tilde{z}_{j}\}_{j = 1}^{m}$ and a \textit{transformation} dataset $\{x_{i}, y_{i}\}_{i = 1}^{n}$. To connect the two \glspl{GP}, we posit that the unobserved targets $\bvec{z}$ at $\bvec{x}$ and at $\bvec{y}$ would have been the same if they were observed. Note that like regular \glspl{GP}, to \glspl{TTGP} the inputs $x$ and $y$ are not modeled as random variables but treated as index variables instead.}
		\vspace{-0.5em}
		\label{fig:ttgpr}
		\end{figure}
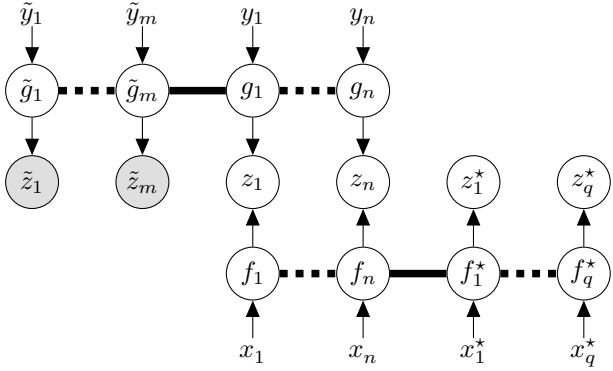
		
		We refer to this setting as \textit{transformed regression}. They can be seen as another dimension of modeling with \gls{LRR} as the base model (\cref{fig:model-cube}). \Gls{KRR} is obtained from \gls{LRR} via the kernel trick and Woodbury identity, and they are \gls{MAP} solutions or predictive means of \gls{GPR} and \gls{BLR} respectively \citep{rasmussen2006gaussian}. Consequently, we also refer to \gls{GPR} as \gls{BKR}. Analogous relationships hold between transformed models. 
				
	\subsection{Task Transformed Regression}
	\label{sec:task_transformed_regression}
	
		We define \textit{\gls{TTR}} as the problem of learning to predict a target variable $Z$ from features $X$ when no direct sample pairs of $X$ and $Z$ are available but instead indirect samples $\{x_{i}, y_{i}\}_{i = 1}^{n}$ and $\{\tilde{y}_{j}, \tilde{z}_{j}\}_{j = 1}^{m}$ with a mediating variable $Y$ are available. The name illustrates the idea of transforming the task of regressing $Z$ on $Y$ to learn $g : \mathcal{Y} \to \mathbb{R}$, using the \textit{task} or \textit{original} dataset $\{\tilde{y}_{j}, \tilde{z}_{j}\}_{j = 1}^{m}$, to the task of regressing $Z$ on $X$ to learn $f : \mathcal{X} \to \mathbb{R}$, by mediating the \textit{task} dataset through the \textit{transformation} dataset $\{x_{i}, y_{i}\}_{i = 1}^{n}$. As the mediating variable $Y$ links the two sets together, $\bvec{y}$ and $\tilde{\bvec{y}}$ control the task transformation.

		\paragraph{\Gls{DME} as solution to chained loss}
		
		We formulate losses for \gls{TTR}, and establish \glspl{DME} as solutions. We begin with the parametric case with $f(x) = \bvec{w}^{T} \phi(x)$ and $g(y) = \bvec{v}^{T} \psi(y)$.
		\begin{theorem}[\Gls{TTLRR}]
		\label{thm:ttlrr}
			The weights of the parametric \gls{DME} estimator $\bar{f}(x) = \bar{\bvec{w}}^{T} \phi(x)$ (\cref{def:parametric_dme_estimator}) solve chained regularized least square losses,
			\vspace{-1em}
			\begin{equation}
			\begin{aligned}
				\hat{\bvec{v}}[\bvec{w}] :=& \; \argmin_{\bvec{v} \in \mathbb{R}^{q}} \frac{1}{n} \sum_{i = 1}^{n} (\bvec{w}^{T} \phi(x_{i}) - \bvec{v}^{T} \psi(y_{i})))^{2} + \lambda \| \bvec{v} \|^{2}, \\
				\bar{\bvec{w}} :=& \; \argmin_{\bvec{w} \in \mathbb{R}^{p}}  \frac{1}{m} \sum_{j = 1}^{m} (\tilde{z}_{j} - \hat{\bvec{v}}[\bvec{w}]^{T} \psi(\tilde{y}_{j}))^{2} + \epsilon \| \bvec{w} \|^{2}.
			\label{eq:ttlrr}
			\end{aligned}
			\end{equation}
		\end{theorem}
		The notation $\hat{\bvec{v}}[\bvec{w}]$ explicitly denotes that $\hat{\bvec{v}}$ depends on $\bvec{w}$. Conceptually, in function space view the first optimization finds $g$ so that $f$ at $\bvec{x}$ best matches with $g$ at $\bvec{y}$, leading to a solution $\hat{g}[f]$ that is dependent on $f$. The second finds $f$ so that $\hat{g}[f]$ at $\tilde{\bvec{y}}$ best matches targets $\tilde{\bvec{z}}$. Using the kernel trick $k(x, x') = \phi(x)^{T} \phi(x')$, we obtain the nonparametric case.
		\begin{lemma}[\Gls{TTKRR}]
		\label{thm:ttkrr}
			The weights of the nonparametric \gls{DME} estimator $\bar{f}(x) = \bar{\bm{\alpha}}^{T} \bm{k}(x)$ (\cref{def:nonparametric_dme_estimator}) satisfies $\bar{\bvec{w}} = \Phi \bar{\bm{\alpha}}$ (the kernel trick).
		\end{lemma}
	
		\paragraph{\Gls{DME} as posterior predictive mean of \gls{TTGP}}
		
		We extend \gls{TTLRR} and \gls{TTKRR} to the Bayesian case. This connection reveals that \gls{TTR} models are transformed regression models with transformations and noise covariances that are \textit{learned}.
		
		In the parametric case, we have \gls{TTBLR}. We place separate independent Gaussian priors $p(\bvec{v}) = \mathcal{N}(\bvec{v}; \bvec{0}, \beta^{2} I)$ and $p(\bvec{w}) = \mathcal{N}(\bvec{w}; \bvec{0}, \gamma^{2} I)$ for $g$ and $f$ respectively. As $z$ is not observed directly from $f$ but only for $g$, we include noise only for observing $g$ to arrive at likelihoods $p(z | \bvec{v}) = \mathcal{N}(z ; \bvec{v}^{T} \psi(y), \sigma^{2})$ and $p(z | \bvec{w}) = \mathcal{N}(z ; \bvec{w}^{T} \phi(x), 0)$ for $g$ and $f$ respectively. 
		
		In the nonparametric case, we have \gls{TTBKR}. We place \gls{GP} priors $g \sim \mathcal{GP}(0, \ell)$ and $f \sim \mathcal{GP}(0, k)$ on the functions directly. Consequently, \gls{TTBKR} is also referred to as \gls{TTGPR}. Similar to \gls{TTBLR}, the likelihoods are  $p(z | g) = \mathcal{N}(z ; g(y), \sigma^{2})$ and $p(z | f) = \mathcal{N}(z ; f(x), 0)$.
		
		\begin{figure*}[!t]
			\centering
			\includegraphics[width=\linewidth]{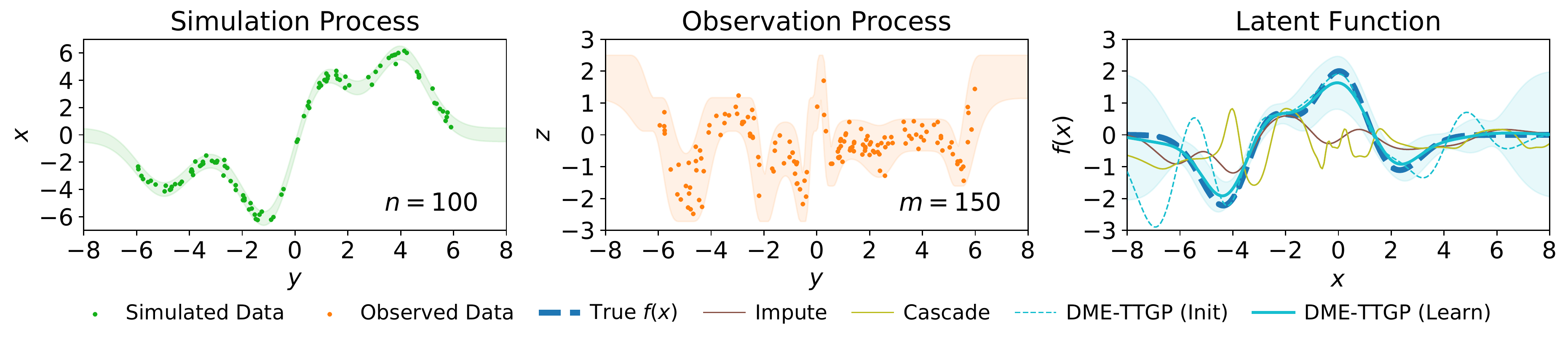}
			\vspace{-2.5em} 
			\caption{Illustration of latent function recovery with a \acrfull{TTGP} on non-trivial simulation and observation processes $p(x|y)$ and $p(z|y)$. (Left) The simulation process $p(x | y)$. (Center) The observation process $p(z | y)$. (Right) The true latent function $f$, the naive solutions using cascaded regressors and imputed data, and the mean and uncertainty bounds of the \gls{TTGP}, also the Bayesian \gls{DME}, with initial and learned hyperparameters. All bounds are 2 standard deviations from the mean.}
		\label{fig:task-transformed-regression}
		\end{figure*}
		The graphical model for \gls{TTGPR} is shown in \cref{fig:ttgpr}. The two \glspl{GP} for $g$ and $f$ are linked by constraining their targets to be the same at $\bvec{y}$ and $\bvec{x}$ respectively. The \gls{GP} for $g$ is used to infer the predictive distribution $p(\tilde{\bvec{z}} | \bvec{z})$, which in turn specifies the overall likelihood $p(\tilde{\bvec{z}} | \bvec{f})$ used to infer $f$. Detailed derivations are provided in the proof of \cref{thm:ttblr_ttbkr}.
			\begin{theorem}[\Acrfull{TTBLR} and \Acrfull{TTBKR}]
			\label{thm:ttblr_ttbkr}
				\textbf{(1)} The \gls{TTBLR} is a \gls{TBLR} with $M = \Psi^{T} (\Psi \Psi^{T} + \frac{\sigma^{2}}{\beta^{2}} I )^{-1} \tilde{\Psi}$ and $\Sigma = \sigma^{2} \tilde{\Psi}^{T} ( \Psi \Psi^{T} + \frac{\sigma^{2}}{\beta^{2}} I )^{-1} \tilde{\Psi} + \sigma^{2} I$ as the transformation and noise covariance. \textbf{(2)} The \gls{TTBKR} is a \gls{TBKR} with transformation $M = (L + \sigma^{2} I)^{-1} \tilde{L}$ and noise covariance $\Sigma = \tilde{\tilde{L}} + \sigma^{2} I - \tilde{L}^{T} (L + \sigma^{2} I)^{-1} \tilde{L}$. \textbf{(3)} The \gls{TTBLR} and \gls{TTBKR} marginal likelihoods are $p(\tilde{\bvec{z}}) = \mathcal{N}(\tilde{\bvec{z}} ; \bvec{0}, [\Sigma^{-1} - \Sigma^{-1} A^{T} \Phi^{T} C \Phi A \Sigma^{-1}]^{-1})$ where $C = [\Phi A \Sigma^{-1} A^{T} \Phi^{T} + \frac{1}{\gamma^{2}} I]^{-1}$ and $p(\tilde{\bvec{z}}) =\mathcal{N}(\tilde{\bvec{z}} ; \bvec{0}, A^{T} K A + \Sigma)$ respectively. \textbf{(4)} For both models, when the posterior for $g$ is approximated via \gls{MAP}, the covariance becomes $\Sigma = \sigma^{2} I$. In this case, the parametric (resp. nonparametric) \gls{DME} estimator (\cref{def:parametric_dme_estimator,,def:nonparametric_dme_estimator}) is the predictive mean of a \gls{TTBLR} (resp. \gls{TTBKR}) with $\lambda = \frac{\sigma^{2}}{n \beta^{2}}$ and $\epsilon = \frac{\sigma^{2}}{m \gamma^{2}}$ (resp. $\lambda = \frac{\sigma^{2}}{n}$ and $\epsilon = \frac{\sigma^{2}}{m}$). An alternative \gls{TTBKR} marginal likelihood is $p(\tilde{\bvec{z}}) = \mathcal{N}(\tilde{\bvec{z}}; \bvec{0}, \sigma^{2} [I - A^{T} (K A A^{T} + \sigma^{2} I)^{-1} K A]^{-1})$. \textbf{(5)} When both posteriors for $g$ and $f$ are approximated via \gls{MAP}, \gls{TTBLR} and \gls{TTBKR} becomes \gls{TTLRR} and \gls{TTKRR} respectively with $\lambda$ and $\epsilon$ from (4). 
			\end{theorem}
		Importantly, our end goal is to infer $f$. While this involves inferring $g$, $g$ is not of direct interest. A simpler alternative is to only perform Bayesian inference on $f$ and approximate $g$ with its \gls{MAP} solution, simplifying the noise covariance via (4) of \cref{thm:ttblr_ttbkr}. This establishes a Bayesian interpretation for \glspl{DME} as \gls{MAP} estimates of \glspl{TTGP}. Critically, by maximizing the \gls{TTGP} marginal likelihood, we can learn \gls{DME} hyperparameters of kernels $k$ and $\ell$, and also $\lambda$ and $\epsilon$. Furthermore, the computational complexity for alternative marginal likelihood is dominated by inversions that are $O(n^{3})$ only, again allowing for larger $m$ at $O(m)$.
		
		In summary, we first establish the \gls{DME} as a nonparametric solution to the \gls{TTR} problem under chained regularized least squares losses that make learning $f$ dependent on learning $g$. While $\lambda$ and $\epsilon$ are previously seen as numerical adjustments to relax \gls{RKHS} assumptions and stabilize matrix inversions in the \gls{KME} framework, they can now be seen as controlling the amount of function regularization under this loss. Secondly, we present \glspl{TTGP} as nonparametric Bayesian solutions to this regression problem and show that \glspl{DME} are their posterior predictive means. Again, inference of $f$ is dependent on the inference of $g$, allowing \gls{GP} uncertainties to propagate through. This connection provides Bayesian interpretations of \glspl{DME} and enable uncertainty estimation in inferring deconditional means. Critically, we use this to derive marginal likelihoods for hyperparameter learning.

\section{Nonparametric Bayes' Rule}
\label{sec:nonparametric_bayes_rule}
	
		\begin{figure*}[!t]
			\centering
			\includegraphics[width=\linewidth]{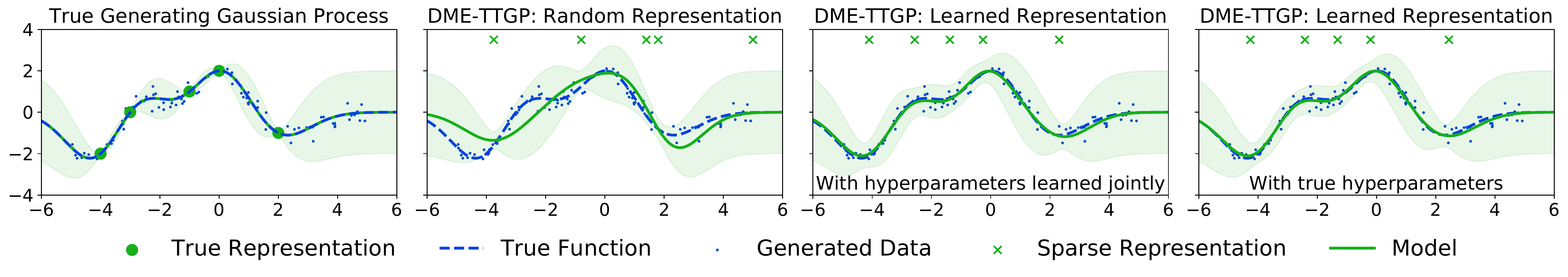}
			\vspace{-2.5em} 
			\caption{Sparse representation learning on a dataset of 100 points generated by using the toy process from \citet{rasmussen2006gaussian} as ground truth. (Left) The true function, represented exactly by a \gls{GP} mean using 5 points. (Left Center) \gls{DME} using 5 random points. (Right Center) \gls{DME} with the 5 points and all hyperparameters learned jointly via its marginal likelihood. (Right) \gls{DME} with the 5 points learned via its marginal likelihood under true hyperparameters. The vertical position of the sparse representation has no meaning.}
		\label{fig:sparse-representation-learning}
		\end{figure*}
			
	While \glspl{DMO} were constructed as solutions to the deconditional mean problem, they also resemble Bayes' rule when we focus solely on considering the encoded relationship between $X$ and $Y$. This was motivated by \cref{thm:dmo}, which revealed that the \gls{DMO} can be fully specified by the \gls{CMO} $C_{X|Y}$ and the second order mean embedding $C_{YY}$ that encoded the likelihood $\mathbb{P}_{Y|X}$ and prior $\mathbb{P}_{Y}$ respectively. To establish this view, we investigate the conditions for which the \gls{DMO} $C_{X|Y}'$ coincide with the \gls{CMO} $C_{Y|X}$ that encodes the posterior $\mathbb{P}_{Y|X}$, leading to a nonparametric Bayes' rule.
		
	While first class citizens of probability rules are density evaluations, first class citizens of the \gls{KME} framework are expectations. Consequently, instead of relating density evaluations, rules under the \gls{KME} framework relate mean embeddings of distributions at various orders. Importantly, while a distribution $Y \sim \mathbb{P}_{Y}$ has one simple density evaluation $p_{Y}(y)$, it can have different \gls{RKHS} representations at different orders such as $\mu_{Y}$ and $C_{YY}$ or higher.
					
	A nonparametric Bayes' rule is a rule which translates Bayes' rule into the \gls{RKHS}, where distributions are represented as \gls{RKHS} operators, alleviating limitations from parametric assumptions such as Gaussian posteriors. It computes a posterior operator $C_{Y|X}$ when given only likelihood operators (e.g. $C_{X|Y}$) and prior operators (e.g. $C_{YY}$). The \gls{DMO} is appealing as all operators involved are of second order and the same second order likelihood and prior operators are used for both the joint and evidence operator.
	
	However, because $C_{XX}'$ is not necessarily the same as $C_{XX}$, the \gls{DMO} $C_{X|Y}' = (C_{XY})^{T} (C_{XX}')^{-1}$ is not necessarily the posterior operator $C_{Y|X} = (C_{XY})^{T} (C_{XX})^{-1}$.  Nevertheless, under certain conditions they coincide with each other.
		\begin{theorem}
		\label{thm:dmo_for_nbr}
			If $C_{X|Y} C_{Y|X} C_{XX} = C_{XX}$, then $C_{XX}' = C_{XX}$ and $C_{X|Y}' = C_{Y|X}$. 
		\end{theorem}
	A special instance where the assumptions are met is when $X = r(Y)$ where $r$ is not necessarily invertible. Importantly, for empirical \glspl{DMO}, having $x_{i} = x_{j}$ for any $y_{i} = y_{j}$ suffices, which can be achieved if all $y_{i}$ are unique. Furthermore, empirical \glspl{DMO} $\bar{C}_{X|Y}'$ can be seen as generalizations of empirical \glspl{CMO} $\hat{C}_{Y|X}$ in the other direction.
		\begin{theorem}
		\label{thm:emp_dmo_special_case}
			If $m = n$ and $\tilde{y}_{i} = y_{i}$ for all $i \in \{1, \dots, n\}$, then the empirical \gls{DMO} corresponding to $\mathbb{P}_{X|Y}$ becomes the empirical \gls{CMO} corresponding to $\mathbb{P}_{Y|X}$ for $\lambda \to 0^{+}$,
			\begin{equation}
			\begin{aligned}
				\lim_{\lambda \to 0^{+}} \bar{C}_{X|Y}' = \Psi \big[ K + n \epsilon I \big]^{-1} \Phi^{T} = \hat{C}_{Y|X}.
			\end{aligned}
			\end{equation}
		\end{theorem}
		\vspace{-1em}
	Intuitively, suppose $\{x_{i}, y_{i}\}_{i = 1}^{n}$ are from $p(x, y) := p(x | y) p(y)$ and $\{\tilde{y}_j\}_{j = 1}^{m} = \{y_i\}_{i = 1}^{n}$ is from $p(y)$, then the \gls{DMO} of $p(x | y)$ is equivalent to the \gls{CMO} of $p(y | x) = p(x | y) p(y) / \int_{\mathcal{Y}} p(x | y) p(y) dy$ in the other direction, as per \cref{thm:emp_dmo_special_case}. In general, however, if $\{x_{i}, y_{i}\}_{i = 1}^{n}$ are from $q(x, y) := p(x | y) q(y)$ and $\{\tilde{y}_j\}_{j = 1}^{m}$ is from $p(y)$, then using the joint samples from $q(x, y)$ only to build the \gls{CMO} will yield the \gls{CMO} of $q(y | x) = p(x | y) q(y) / \int_{\mathcal{Y}} p(x | y) q(y) dy$, while using both the joint samples from $q(x, y)$ and marginal samples from $p(y)$ to build the \gls{DMO} will yield the \gls{CMO} corresponding to $p(y | x) = p(x | y) p(y) / \int_{\mathcal{Y}} p(x | y) p(y) dy$. \Cref{sec:dmo_for_nbr} details further parallels with probabilistic rules.
			
\section{Related Work}
\label{sec:related_work}

	\Glspl{DMO} have strong connections to \gls{KBR}. Both provide a nonparametric Bayes' rule under the \gls{KME} framework. In contrast to \glspl{DMO} where both likelihood and prior operators are of second order, both \gls{KBR}(a) and \gls{KBR}(b) \citep{song2013kernel,fukumizu2013kernel} use a third order likelihood operator $C_{XX|Y}$ and a first order prior embedding $\mu_{Y}$ for the evidence operator $C_{XX} = C_{XX|Y} \mu_{Y}$. \Gls{KBR}(b) further uses a different third order likelihood operator $C_{XY|Y}$ for the joint operator $C_{XY} = C_{XY|Y} \mu_{Y}$. Consequently, \gls{KBR} becomes sensitive to inverse regularizations and effects of prior samples $\tilde{\bvec{y}}$ can vanish. For instance, when $\epsilon \to 0^{+}$, \gls{KBR}(b) degenerate to $\bar{C}_{Y|X} = \Psi K^{-1} \Phi^{T}$, which is a \gls{CMO} that no longer depend on $\tilde{\bvec{y}}$. Instead, \glspl{DMO} degenerate to $\bar{C}_{X|Y}' = \tilde{\Psi} A^{T} \big[A A^{T}\big]^{-1} K^{-1} \Phi^{T}$, retaining their original structure. Detailed comparisons are provided in \cref{sec:connections_to_kbr}.
	
	Viewing \glspl{KME} as regressors provides valuable insights and interpretations to the framework. \Glspl{CMO} can be established as regressors where the vector-valued targets are also kernel induced features \citep{grunewalder2012conditional}. In contrast, we establish \glspl{DMO} as solutions to task transformed regressors which recover latent functions that, together with a likelihood, governs interactions between three variables.
	
	The \gls{TTR} problem describes the setting of learning from conditional distributions in the extreme case where only one sample of $x_{i}$ is available for each $y_{i}$ to describe $p(x | y)$. Dual \glspl{KME} \citep{dai2017learning} formulate this setting as a saddle point problem, and employ stochastic approximations to efficiently optimize over the function space. However, without connections to Bayesian models such as \glspl{TTGP} that admit a marginal likelihood, hyperparameter selection often require inefficient grid search.
	
		\begin{figure*}[!t]
			\centering
			\includegraphics[width=0.49\linewidth]{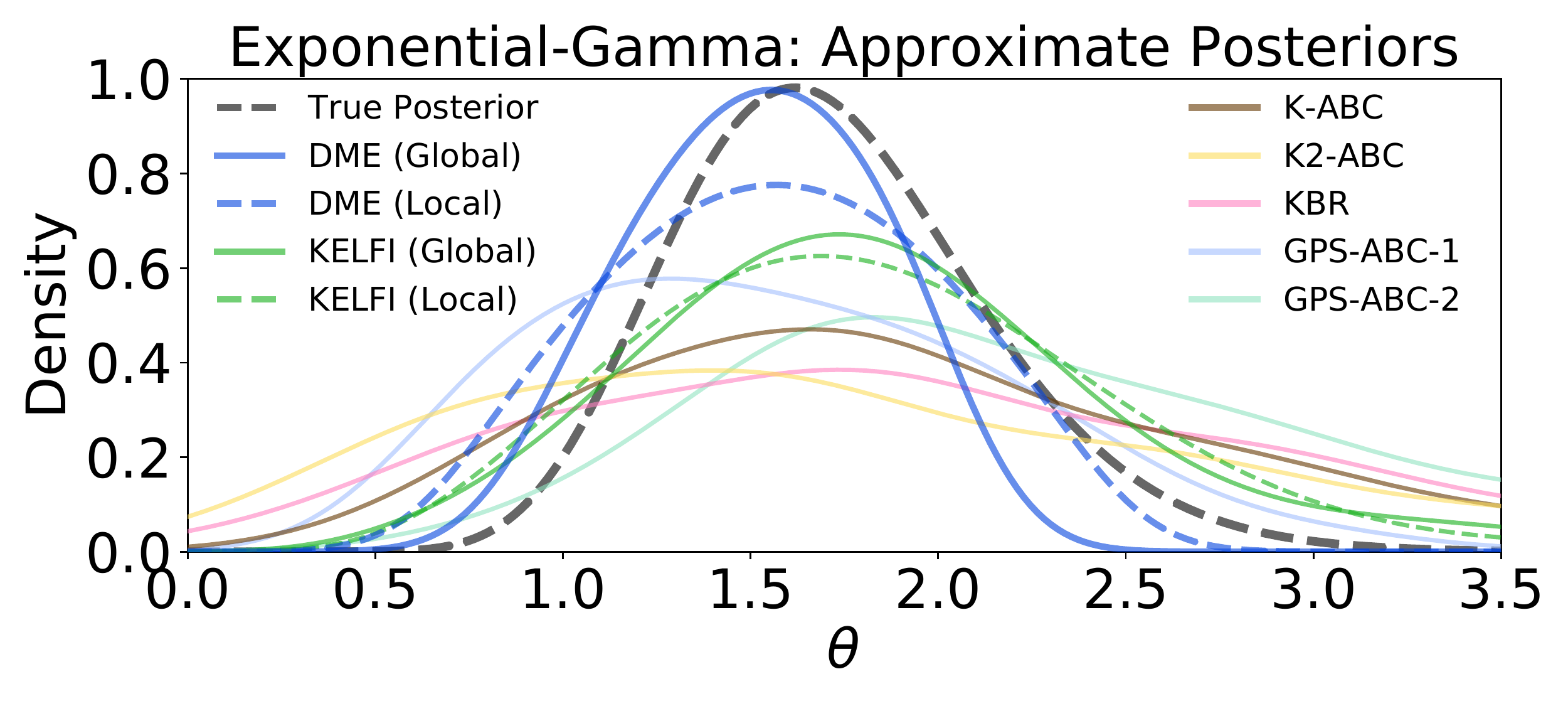}
			\includegraphics[width=0.49\linewidth]{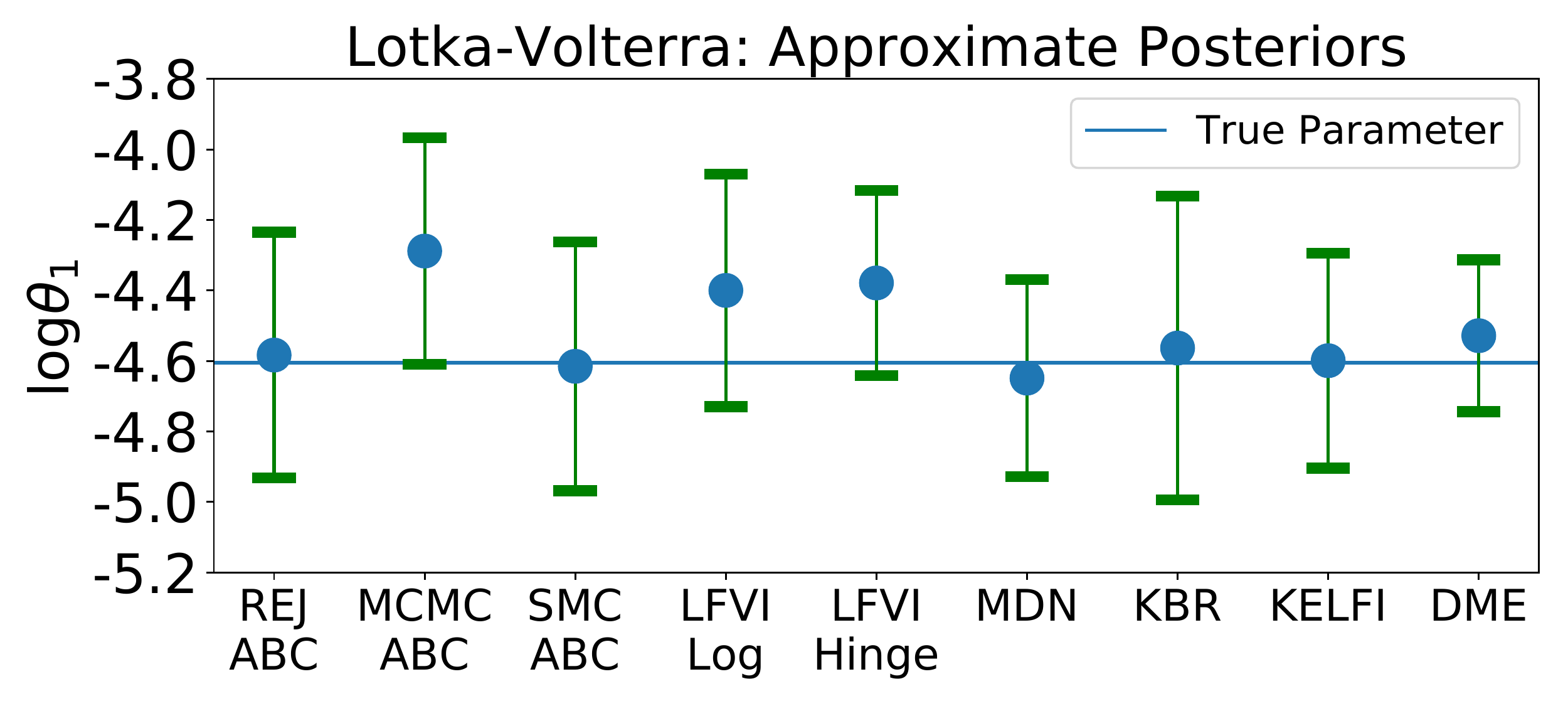}
			\vspace{-1em} 
			\caption{Application to \acrshort{LFI}. (Left) Approximate posteriors under kernel based \acrshort{LFI} methods for the toy exponential-gamma problem using $100$ simulations. `Global' and `Local' refer to the optimality of model hyperparameters with respect to their respective approximate marginal likelihoods. (Right) Approximate posteriors of the first (log) parameter. Error bars represent the middle 95\% credible interval.}
			\vspace{-0.5em} 
		\label{fig:lfi}
		\end{figure*}
		
	Hyperparameter learning of marginal embeddings have been investigated by placing \gls{GP} priors on the embedding itself to yield a marginal likelihood objective \citep{flaxman2016bayesian}. However, it is unclear how this can be extended to \glspl{CME}. Our marginal likelihoods (\cref{thm:ttblr_ttbkr} and \cref{thm:aml_lfi}) provide such objective for \glspl{DME} and, due to \cref{thm:emp_dmo_special_case}, it can also be applied to \glspl{CME} as a special case.

\section{Applications and Experiments}
\label{sec:experiments}

	While \glspl{DME} are developed to complement the theoretical framework of \glspl{KME}, in this section we describe and demonstrate some of their practical applications with experiments.
	
	\subsection{Hyperparameter Learning for \gls{TTR}}
	
		We first illustrate in \cref{fig:task-transformed-regression} the \gls{TTR} problem, the primary application of \glspl{TTGP} and \glspl{DME}. While $X$ and $Y$ are multivariate in general, we use 1D examples to enable visualizations. Although this is a 1D problem, the simulation process $p(x | y)$ and observation process $p(z | y)$ are governed by non-trivial relationships where successful recovery of $f$ requires dealing with difficult multi-modalities in $p(y | x)$. To generate the data, we choose non-trivial functions $r$ and $f$ and generate $X_{i} = r(Y_{i}) + \eta_{i}$ and $\tilde{Z}_{j} = f(r(\tilde{Y}_{j}) + \tilde{\eta}_{j}) + \tilde{\xi}_{j}$, where $Y_{i}, \tilde{Y}_{j} \sim U(-6, 6)$ and $ \eta_{i}, \tilde{\eta}_{j}, \tilde{\xi}_{j} \sim \mathcal{N}(0, 0.25^{2})$ for all $i \in \{1, \dots, n\}$ and $j \in \{1, \dots, m\}$. In this way, $p(x | y) = \mathcal{N}(x; r(y), 0.25^{2})$, $p(z | x) = \mathcal{N}(z; f(x), 0.25^{2})$, and $\mathbb{E}[Z | Y = y] = \mathbb{E}[f(r(y) + \tilde{\eta}) + \tilde{\xi}] = \mathbb{E}[f(X) | Y = y]$.
		
		By optimizing the marginal likelihood in \cref{thm:ttblr_ttbkr} (3), we see that the \gls{DME} is able to adapt from its initial hyperparameters to learn the latent function accurately. 
		
		We compare this to two naive solutions that one may propose when faced with a \gls{TTR} problem. The \textit{cascade} method trains separate regressors from $X$ to $Y$, with the transformation set, and from $Y$ to $Z$, with the task set. They use the former to predict $y^{\star}$ from $x^{\star}$ and the latter to predict $z^{\star}$ from $y^{\star}$. The \textit{impute method} trains a regressor from $Y$ to $Z$ with the task set and predicts $\bvec{z}_{\mathrm{fake}}$ at locations $\bvec{y}$, and trains a regressor on the dataset $(\bvec{x}, \bvec{z}_{\mathrm{fake}})$ to predict $z^{\star}$ from a new $x^{\star}$. We use \gls{GPR} means (\gls{KRR}) for all such regressors. Both methods suffer because uncertainty propagation is lost by training regressors separately. The cascade method suffers further because $p(y | x)$ is usually highly multi-modal such as in this example, so unimodal regressors like \gls{GPR} from $X$ to $Y$ are unsuitable. This also highlights that while \glspl{DME} provides unimodal Gaussian uncertainty on function evaluations and thus $Z$, they capture multi-modality as a nonparametric Bayes' rule between $X$ and $Y$.
	
	\subsection{Sparse Representation Learning with \gls{TTGP}}
	
		A special case of the \gls{TTR} problem is to learn sparse representations for big data with trainable inducing points. Continuing with the notations used so far, we are given a large original dataset $(\tilde{\bvec{y}}, \tilde{\bvec{z}})$ of size $m$, with inputs $Y$ and target $Z$. We let the transformation dataset be a set of $n$ inducing points $\bvec{x} = \bvec{y} = \bvec{u}$ for $Y$ where $n << m$. That is, we degenerate to $X = Y$ and $\mathcal{X} = \mathcal{Y}$. We maximize the alternative marginal likelihood in \cref{thm:ttblr_ttbkr} (4) with respect to the inducing points and learn the \gls{TTGP} hyperparameters jointly. For predictive mean we use the alternative computational form in \cref{def:nonparametric_dme_estimator}. Similar form exists for the covariance. These alternative forms are suitable for this application because $n$ is small for its $O(n^{3})$ inversions and dependence on $m$ is only $O(m)$. We illustrate this process in \cref{fig:sparse-representation-learning}. 

	\subsection{Likelihood-Free Inference with \gls{DME}}
	\label{sec:experiments_likelihood_free_inference}
	
		As a nonparametric Bayes' rule, \glspl{DME} can be used for \gls{LFI} \citep{marin2012approximate} where likelihood evaluations are intractable but sampling from a simulator $\bvec{x} \sim p(\bvec{x} | \bm{\theta})$ is possible. The simulator takes parameters $\bm{\theta}$ and stochastically generates simulated data that are often summarized into statistics $\bvec{x}$. Observed data are also summarized into statistics $\bvec{y}$, and discrepancies with $\bvec{x}$ are often measured by an $\epsilon$-kernel $\kappa_{\epsilon}(\bvec{y}, \bvec{x}) = p_{\epsilon}(\bvec{y} | \bvec{x})$ such that $p_{\epsilon}(\bvec{y} | \bm{\theta}) = \int p_{\epsilon}(\bvec{y} | \bvec{x}) p(\bvec{x} | \bm{\theta}) d\bm{\theta}$. This $\epsilon$ is not to be confused with the regularization used for $C_{XX}'$, which we denote as $\delta$ for this section only. After selecting a prior $p(\bm{\theta})$, the goal is to approximate the posterior $p_{\epsilon}(\bm{\theta} | \bvec{y})$.
		
		Translating notations into the \gls{LFI} setting, we have $x_{i} \rightarrow \bvec{x}_{i}$, $y_{i} \rightarrow \bm{\theta}_{i}$, $\tilde{y}_{j} \rightarrow \tilde{\bm{\theta}_{j}}$, and $x \rightarrow \bvec{y}$. We first simulate $\bvec{x}_{i} \sim p(\bvec{x} | \bm{\theta}_{i})$ on parameters $\{\bm{\theta}_{i}\}_{i = 1}^{n} \sim \pi(\bm{\theta})$ not necessarily from the prior to get $\{\bm{\theta}_{i}, \bvec{x}_{i}\}_{i = 1}^{n}$ for the likelihood, and sample $\{\tilde{\bm{\theta}}_{j}\}_{j = 1}^{m}  \sim p(\bm{\theta})$ for the prior. We then build the \gls{DME} $\bar{\mu}_{\bm{\Theta} | \bvec{X} = \bvec{y}}$ and sample it with kernel herding \citep{chen2010super} for posterior super-samples. This is described in \cref{alg:dme-lfi}. We also provide the approximate marginal likelihood objective $\bar{q}$ to maximize for hyperparameter learning of the \gls{DME}. Derivations are detailed in \cref{sec:experiment_details}.
		
			\begin{algorithm}
				\caption{\Acrlongpl{DME} for \gls{LFI}}
				\label{alg:dme-lfi}
				\begin{algorithmic}[1]
					\STATE {\bfseries Input:} Data $\bvec{y}$, simulations $\{\bm{\theta}_{i}, \bvec{x}_{i}\}_{i = 1}^{n} \sim p(\bvec{x} | \bm{\theta}) \pi(\bm{\theta})$, prior samples $\{\tilde{\bm{\theta}}_{j}\}_{j = 1}^{m} \sim p(\bm{\theta})$, query points $\{\bm{\theta}^{\star}_{r}\}_{r = 1}^{R}$, kernels $k$, $\kappa_{\epsilon}$, $\ell$, and $\ell'$, regularization $\lambda$ and $\delta$
					\STATE $L \leftarrow \{\ell(\bm{\theta}_{i}, \bm{\theta}_{j})\}_{i, j = 1}^{n, n}$, $\tilde{L} \leftarrow \{\ell\bm{\theta}_{i}, \tilde{\bm{\theta}}_{j})\}_{i, j = 1}^{n, m}$
					\STATE $A \leftarrow (L + n \lambda I)^{-1} \tilde{L}$, $\tilde{L}^{\star} \leftarrow \{\ell(\tilde{\bm{\theta}}_{j}, \bm{\theta}^{\star}_{r})\}_{j, r = 1}^{m, R}$, \STATE $K \leftarrow \{k(\bvec{x}_{i}, \bvec{x}_{j})\}_{i, j = 1}^{n, n}$, $\bm{k}(\bvec{y}) \leftarrow \{k(\bvec{x}_{i}, \bvec{y})\}_{i = 1}^{n}$
					\STATE \gls{DME}: $\bm{\mu} \leftarrow (\tilde{L}^{\star})^{T}  A^{T} \big[K A A^{T} + m \delta I \big]^{-1} \bm{k}(\bvec{y}) \in \mathbb{R}^{R}$
					\FOR{$s \in \{1, \dots, S\}$ with $\bvec{a} \leftarrow \bvec{0} \in \mathbb{R}^{R}$ initialized} 
					\STATE $\hat{\bm{\theta}}_{s} \leftarrow \bm{\theta}^{\star}_{r^{\star}}$ where $r^{\star} \leftarrow \argmax_{r} \mu_{r} - (a_{r} / s)$
					\STATE $\bvec{a} \leftarrow \bvec{a} + \{\ell'(\bm{\theta}^{\star}_{r}, \hat{\bm{\theta}}_{s}) \}_{r = 1}^{R}$
					\ENDFOR
					\STATE {\bfseries Output:} Posterior super-samples $\{ \hat{\bm{\theta}}_{s} \}_{s = 1}^{S}$
					\STATE {\bfseries Learning:} $\bar{q} \leftarrow \mathrm{mean}(A^{T} \bm{\kappa}_{\epsilon})$, $\bm{\kappa_{\epsilon}} \leftarrow \{\kappa_{\epsilon}(\bvec{y}, \bvec{x}_{i})\}_{i = 1}^{n}$
				\end{algorithmic}
			\end{algorithm}

		\Cref{fig:lfi} demonstrates \cref{alg:dme-lfi} on two standard benchmarks. For the toy exponential-gamma problem we compare directly with other kernel approaches. As simulations are usually very expensive, we show the case with very limited simulations $(n = 100)$, leading to most methods producing posteriors wider than the ground truth. Nevertheless, by optimizing $\bar{q}$ in line 11, \glspl{DME} can adapt their kernel length scales accordingly. For Lotka-Volterra, the ABC methods used more than 100000 simulations, while MDN used 10000 simulations. To achieve competitive accuracy, kernel approaches such as \glspl{DME}, \acrshort{KELFI} \citep{hsu2019bayesian}, and \gls{KBR} used 2000, 2500, and 2500 simulations. Acronyms and experimental details are described in \cref{sec:experiment_details}.
		
\section{Conclusion}
\label{sec:conclusion}

	The connections of \glspl{DME} with \glspl{CME} and \glspl{GP} produce useful insights towards the \gls{KME} framework, and are important steps towards establishing Bayesian views of \glspl{KME}. \Glspl{DME} provide novel solutions to a class of nonparametric Bayesian regression problems and enable applications such as sparse representation learning and \gls{LFI}. For future work, relaxing assumptions required for \glspl{DMO} as a nonparametric Bayes' rule can have fruitful theoretical and practical implications.
	
\bibliography{references}
\bibliographystyle{icml2019}

\onecolumn
\newpage
\appendix
\section{Supporting Proofs for \Cref{sec:conditional_mean_embedding}}
\label{sec:conditional_mean_embedding_proofs}

	\begin{proof}[\textbf{Proof of \Cref{thm:cmo}}]
		Let $f \in \mathcal{H}_{k}$ and $g(y) := \mathbb{E}[f(X) | Y = y]$. Assuming $g \in \mathcal{H}_{\ell}$, then $C_{YY} g = C_{YX} f$ \citep{fukumizu2004dimensionality}, so that
		\begin{equation}
		\begin{aligned}
			g &= C_{YY}^{-1} C_{YX} f \\
			&= ( (C_{YX})^{T} C_{YY}^{-1} )^{T} f \\
			&= ( C_{XY} C_{YY}^{-1} )^{T} f,
		\end{aligned}
		\end{equation}
		where the inverse $C_{YY}^{-1}$ exists because $\ell(y, \cdot)$ is assumed to be in the image of $C_{YY}$ so that any $g \in \mathcal{H}_{\ell}$ is also in the image. Hence, $C_{XY} C_{YY}^{-1}$ satisfies the definition of a \gls{CMO}.
	\end{proof}
		
	\begin{proof}[\textbf{Proof of \Cref{thm:cmo_definition_equivalence}}]
		Each of the following statements are equivalent to each other.
		\begin{equation}
		\begin{aligned}
			(C_{X|Y})^{T} f &= \mathbb{E}[f(X) | Y = \cdot], \quad \forall f \in \mathcal{H}_{k} \\
			\iff \langle \ell(y, \cdot), (C_{X|Y})^{T} f \rangle_{\mathcal{H}_{\ell}} &= \langle \ell(y, \cdot), \mathbb{E}[f(X) | Y = \cdot] \rangle_{\mathcal{H}_{\ell}}, \quad \forall f \in \mathcal{H}_{k}, \quad \forall y \in \mathcal{Y} \\
			\iff \langle C_{X|Y} \ell(y, \cdot), f \rangle_{\mathcal{H}_{k}} &= \mathbb{E}[f(X) | Y = y] = \langle \mathbb{E}[k(X, \cdot) | Y = y], f \rangle_{\mathcal{H}_{k}}, \quad \forall f \in \mathcal{H}_{k}, \quad \forall y \in \mathcal{Y} \\
			\iff C_{X|Y} \ell(y, \cdot) &= \mathbb{E}[k(X, \cdot) | Y = y], \quad \forall y \in \mathcal{Y}.
		\end{aligned}
		\end{equation}
		Consequently, the first and last statements are equivalent.
	\end{proof}

	\begin{proof}[\textbf{Proof of \Cref{thm:emp_cmo}}]

			We show that the empirical \gls{CMO} can be written as \eqref{eq:emp_cmo}. We use a special case of the Woodbury identity \citep{higham2002accuracy}, $B (CB + \lambda I)^{-1} = (BC + \lambda I)^{-1} B$, where $B$ and $C$ are appropriately defined operators, such matrices with the correct shapes. Using the empirical forms for the cross-covariance operators, we have
			\begin{equation}
			\begin{aligned}
				\hat{C}_{X|Y} :=&\; \hat{C}_{XY} (\hat{C}_{YY} + \lambda I)^{-1} \\
				=&\; \frac{1}{n} \Phi \Psi^{T} (\frac{1}{n} \Psi \Psi^{T} + \lambda I)^{-1} \\
				=&\; \Phi \Psi^{T} (\Psi \Psi^{T} + n \lambda I)^{-1} \\
				=&\; \Phi (\Psi^{T} \Psi + n \lambda I)^{-1} \Psi^{T} \\
				=&\; \Phi (L + n \lambda I)^{-1} \Psi^{T}.
			\end{aligned}
			\end{equation}
	\end{proof}
		
\newpage	
\section{Supporting Proofs for \Cref{sec:deconditional_mean_embedding}}
\label{sec:deconditional_mean_embedding_proofs}

	\begin{proof}[\textbf{Proof of \Cref{thm:dmo}}]
		Let $f \in \mathcal{H}_{k}$ and $g(y) := \mathbb{E}[f(X) | Y = y]$, then from \cref{def:cmo} we have
		\begin{equation}
		\begin{aligned}
			g &= (C_{X|Y})^{T} f \\
			C_{XY} g &= C_{XY} (C_{X|Y})^{T} f \\
			C_{X|Y} C_{YY} g &= C_{X|Y} C_{YY} (C_{X|Y})^{T} f \\
			(C_{X|Y} C_{YY} (C_{X|Y})^{T})^{-1} C_{X|Y} C_{YY} g &= f \\
			( (C_{X|Y} C_{YY})^{T} (C_{X|Y} C_{YY} (C_{X|Y})^{T})^{-1} )^{T} g &= f,
		\end{aligned}
		\end{equation}
		where the inverse $(C_{X|Y} C_{YY} (C_{X|Y})^{T})^{-1}$ exists because $C_{XY} g \in \mathcal{H}_{k}$ and $k(x, \cdot) \in \mathrm{image}(C_{X|Y} C_{YY} (C_{X|Y})^{T})$ so that $C_{XY} g$ for any $g \in \mathcal{H}_{\ell}$ is also in the image. In the last line we also used the fact that $(C_{X|Y} C_{YY} (C_{X|Y})^{T})^{T} = C_{X|Y} C_{YY} (C_{X|Y})^{T}$ is symmetric since $(C_{YY})^{T} = C_{YY}$. Hence, $(C_{X|Y} C_{YY})^{T} (C_{X|Y} C_{YY} (C_{X|Y})^{T})^{-1}$ satisfies the definition of a \gls{DMO}. The assumption $\ell(y, \cdot) \in \mathrm{image}(C_{YY})$ is required so that the original \gls{CMO} exists and is unique.		
	\end{proof}
	
	\begin{proof}[\textbf{Proof of \Cref{thm:dmo_is_pseudo_inv_of_cmo}}]
		Since $\ell(y, \cdot) \in \mathrm{image}(C_{YY})$ for all $y \in \mathcal{Y}$ and $k(x, \cdot) \in \mathrm{image}(C_{X|Y} C_{YY} (C_{X|Y})^{T})$ for all $x \in \mathcal{X}$, we have that $C_{YY}^{-1}$ exists so that $C_{X|Y}$ is unique and $(C_{X|Y} C_{YY} (C_{X|Y})^{T})^{-1}$ exists so that $C_{X|Y}'$ is unique. Due to \cref{thm:dmo} we have $C_{X|Y}' = (C_{X|Y} C_{YY})^{T} (C_{X|Y} C_{YY} (C_{X|Y})^{T})^{-1}$. Since $C_{X|Y} C_{YY} (C_{X|Y})^{T}$ is at least positive semi-definite and invertible we can write $(C_{X|Y} C_{YY} (C_{X|Y})^{T})^{-1} = \lim_{\epsilon \to 0^{+}} (C_{X|Y} C_{YY} (C_{X|Y})^{T} + \epsilon I)^{-1}$, 
		\begin{equation}
		\begin{aligned}
			C_{X|Y}'
			&= \lim_{\epsilon \to 0^{+}} (C_{X|Y} C_{YY})^{T} (C_{X|Y} C_{YY} (C_{X|Y})^{T} + \epsilon I)^{-1} \\
			&= \lim_{\epsilon \to 0^{+}} (C_{X|Y} C_{YY})^{T} (C_{X|Y} (C_{X|Y} C_{YY})^{T} + \epsilon I)^{-1} \\
			&= \lim_{\epsilon \to 0^{+}} (C_{XY})^{T} (C_{X|Y} (C_{XY})^{T} + \epsilon I)^{-1} \\
			&= \lim_{\epsilon \to 0^{+}} C_{YX} (C_{X|Y} C_{YX} + \epsilon I)^{-1} \\
			&= \lim_{\epsilon \to 0^{+}} (C_{YX} C_{X|Y} + \epsilon I)^{-1} C_{YX} \\
			&= \lim_{\epsilon \to 0^{+}} (C_{YY} C_{YY}^{-1} C_{YX} C_{X|Y} + \epsilon C_{YY} C_{YY}^{-1} )^{-1} C_{YX} \\
			&= \lim_{\epsilon \to 0^{+}} (C_{YY}^{-1} C_{YX} C_{X|Y} + \epsilon C_{YY}^{-1})^{-1} C_{YY}^{-1} C_{YX} \\
			&= \lim_{\epsilon \to 0^{+}} (C_{YY}^{-1} C_{YX} C_{X|Y} + \epsilon C_{YY}^{-1})^{-1} C_{YY}^{-1} C_{YX} \\
			&= \lim_{\epsilon \to 0^{+}} ((C_{XY} C_{YY}^{-1})^{T} C_{X|Y} + \epsilon C_{YY}^{-1})^{-1} (C_{XY} C_{YY}^{-1})^{T} \\
			&= \lim_{\epsilon \to 0^{+}} ((C_{X|Y})^{T} C_{X|Y} + \epsilon C_{YY}^{-1})^{-1} (C_{X|Y})^{T} \\
			&= ((C_{X|Y})^{T} C_{X|Y})^{-1} (C_{X|Y})^{T} =: C^{\dagger}_{X|Y}.
		\end{aligned}
		\end{equation}
		In line 6 we used the Woodbury identity \citep{higham2002accuracy}. In the last line,  the limit exists as $((C_{X|Y})^{T} C_{X|Y})^{-1}$ exists.
	\end{proof}
	
	\begin{proof}[\textbf{Proof of \Cref{thm:emp_dmo}}]
			We show that the empirical \gls{DMO} can be written as \eqref{eq:emp_dmo}. From \cref{def:emp_cmo} and \cref{thm:emp_cmo}, the likelihood operator is estimated from $\{x_{i}, y_{i}\}_{i = 1}^{n}$ as
			\begin{equation}
				\hat{C}_{X|Y} := \hat{C}_{XY} (\hat{C}_{XX} + \lambda I)^{-1} = \Phi (L + n \lambda I)^{-1} \Psi^{T}.
			\end{equation}
			The prior operator corresponding to the marginal $\mathbb{P}_{Y}$ is estimated from $\{\tilde{y}_{j}\}_{j = 1}^{m}$ as
			\begin{equation}
				\tilde{C}_{YY} = \frac{1}{m} \tilde{\Psi} \tilde{\Psi}^{T}.
			\end{equation}
			Let $A := (L + n \lambda I)^{-1} \tilde{L}$, the joint operator is estimated as
			\begin{equation}
				\hat{C}_{X|Y} \tilde{C}_{YY} = \frac{1}{m} \Phi (L + n \lambda I)^{-1} \Psi^{T} \tilde{\Psi} \tilde{\Psi}^{T} = \frac{1}{m} \Phi (L + n \lambda I)^{-1} \tilde{L} \tilde{\Psi}^{T} = \frac{1}{m} \Phi A \tilde{\Psi}^{T}.
			\end{equation}
			The evidence operator is estimated as
			\begin{equation}
			\begin{aligned}
				\hat{C}_{X|Y} \tilde{C}_{YY} (\hat{C}_{X|Y})^{T} &= \frac{1}{m} \Phi (L + n \lambda I)^{-1} \tilde{L} \tilde{\Psi}^{T} \Psi (L + n \lambda I)^{-1} \Phi^{T} \\
				&= \frac{1}{m} \Phi (L + n \lambda I)^{-1} \tilde{L} \tilde{L}^{T} (L + n \lambda I)^{-1} \Phi^{T} \\
				&= \frac{1}{m} \Phi A A^{T} \Phi^{T}.
			\end{aligned}
			\end{equation}
			Finally, by \cref{def:emp_dmo}, the \gls{DMO} is estimated as
			\begin{equation}
			\begin{aligned}
				\bar{C}_{X|Y}' &= (\hat{C}_{X|Y} \tilde{C}_{YY})^{T} (\hat{C}_{X|Y} \tilde{C}_{YY} (\hat{C}_{X|Y})^{T} + \epsilon I)^{-1} \\
				&= \bigg[ \frac{1}{m} \Phi A \tilde{\Psi}^{T} \bigg]^{T} \bigg[ \frac{1}{m} \Phi A A^{T} \Phi^{T} + \epsilon I \bigg]^{-1} \\
				&= \big[ \Phi A \tilde{\Psi}^{T} \big]^{T} \big[ \Phi A A^{T} \Phi^{T} + m \epsilon I \big]^{-1} \\
				&= \tilde{\Psi} A^{T} \Phi^{T} \big[ \Phi A A^{T} \Phi^{T} + m \epsilon I \big]^{-1} \\
				&= \tilde{\Psi} \big[ A^{T} \Phi^{T} \Phi A + m \epsilon I \big]^{-1} A^{T} \Phi^{T} \\
				&= \tilde{\Psi} \big[ A^{T} K A + m \epsilon I \big]^{-1} A^{T} \Phi^{T}.
			\end{aligned}
			\end{equation}
	\end{proof}
	
\newpage
\section{Supporting Proofs for \Cref{sec:task_transformed_gaussian_process}}
\label{sec:task_transformed_gaussian_process_proofs}

	\begin{proof}[\textbf{Proof of \Cref{thm:ttlrr}}]
		Each optimization is a standard regularized least squares problem. The first optimization over $\bvec{v}$ can be written as 
			\begin{equation}
			\begin{aligned}
				&\hat{\bvec{v}}[\bvec{w}] = \argmin_{\bvec{v} \in \mathbb{R}^{q}} \| \Phi^{T} \bvec{w} - \Psi^{T} \bvec{v} \|^{2} + n \lambda \| \bvec{v} \|^{2},
			\end{aligned}
			\end{equation}
		where $\bvec{f} = \Phi^{T} \bvec{w}$ is the target and $\Psi$ is the feature matrix. This gives the solution $\hat{\bvec{v}}[\bvec{w}] = (\Psi \Psi^{T} + n \lambda I)^{-1} \Psi (\Phi^{T} \bvec{w})$. Therefore, The second optimization over $\bvec{w}$ can be written as
			\begin{equation}
			\begin{aligned}
				\bar{\bvec{w}} &= \argmin_{\bvec{w} \in \mathbb{R}^{p}}  \| \tilde{\bvec{z}} - \tilde{\Psi}^{T} \hat{\bvec{v}}[\bvec{w}] \|^{2} + m \epsilon \| \bvec{w} \|^{2} \\
				&= \argmin_{\bvec{w} \in \mathbb{R}^{p}}  \| \tilde{\bvec{z}} - \tilde{\Psi}^{T} (\Psi \Psi^{T} + n \lambda I)^{-1} \Psi \Phi^{T} \bvec{w} \|^{2} + m \epsilon \| \bvec{w} \|^{2} \\
				&= \argmin_{\bvec{w} \in \mathbb{R}^{p}}  \| \tilde{\bvec{z}} - A^{T} \Phi^{T} \bvec{w} \|^{2} + m \epsilon \| \bvec{w} \|^{2} \\
				&= \argmin_{\bvec{w} \in \mathbb{R}^{p}}  \| \tilde{\bvec{z}} - \Theta^{T} \bvec{w} \|^{2} + m \epsilon \| \bvec{w} \|^{2},
			\end{aligned}
			\end{equation}
		where we used $A := \Psi^{T} (\Psi \Psi^{T} + n \lambda I)^{-1} \tilde{\Psi}$ as per \cref{def:parametric_dme_estimator} and we define $\Theta := \Phi A$. This is now a regularized least squares problem with $\tilde{\bvec{z}}$ as the target and $\Theta := \Phi A$ as the feature matrix. This gives the solution $\bar{\bvec{w}} = (\Theta \Theta^{T} + m \epsilon I)^{-1} \Theta \tilde{\bvec{z}} = (\Phi A A^{T} \Phi^{T} + m \epsilon I)^{-1} \Phi A \tilde{\bvec{z}}$, which yields the parametric \gls{DME} estimator in \cref{def:parametric_dme_estimator}.
	\end{proof}
	
	\begin{proof}[\textbf{Proof of \Cref{thm:ttkrr}}]
		We first establish that the transformation matrix in \cref{def:nonparametric_dme_estimator} $A = (L + n \lambda I)^{-1} \tilde{L}$ is the same as the transformation matrix in \cref{def:parametric_dme_estimator} $A = \Psi^{T} (\Psi \Psi^{T} + n \lambda I)^{-1} \tilde{\Psi}$ via a special case of the Woodbury identity $B (CB + \delta I)^{-1} = (BC + \delta I)^{-1} B$ for appropriately sized matrices or operators $B$ and $C$ \citep{higham2002accuracy}. Consequently, $(L + n \lambda I)^{-1} \tilde{L} = (\Psi^{T} \Psi + n \lambda I)^{-1} \Psi^{T} \tilde{\Psi} = \Psi^{T} (\Psi \Psi^{T} + n \lambda I)^{-1} \tilde{\Psi}$.
		
		From \cref{def:nonparametric_dme_estimator} we have $\bar{\bm{\alpha}} := A \big[ A^{T} K A + m \epsilon I \big]^{-1} \tilde{\bvec{z}}$ so that
		\begin{equation}
		\begin{aligned}
			\Phi \bar{\bm{\alpha}} &= \Phi A \big[ A^{T} K A + m \epsilon I \big]^{-1} \tilde{\bvec{z}} \\
			&= \Phi A \big[ A^{T} \Phi^{T} \Phi A + m \epsilon I \big]^{-1} \tilde{\bvec{z}} \\
			&= \big[ \Phi A A^{T} \Phi^{T} + m \epsilon I \big]^{-1} \Phi A \tilde{\bvec{z}} \\
			&= \bar{\bvec{w}}.
		\end{aligned}
		\end{equation}
		This relationship is a direct consequence of the kernel trick, where we used $k(x, x') = \phi(x)^{T} \phi(x')$ such that $K = \Phi^{T} \Phi$.
	\end{proof}		
	
	\begin{proof}[\textbf{Proof of \Cref{thm:ttblr_ttbkr} Part 1} -- \Acrfull{TTBLR}]
		In this proof we provide the derivations for \acrfull{TTBLR}. We first reiterate the priors and likelihoods used.
		
		\paragraph{Priors}

			We first place priors on the weights of our linear models $g(y) = \bvec{v}^{T} \psi(y)$ and $f(x) = \bvec{w}^{T} \psi(x)$,
			\begin{equation}
			\begin{aligned}
				p(\bvec{v}) &\sim \mathcal{N}(\bvec{v}; \bvec{0}, \beta^{2} I), \\
				p(\bvec{w}) &\sim \mathcal{N}(\bvec{w}; \bvec{0}, \gamma^{2} I).
			\end{aligned}
			\end{equation}
			
		\paragraph{Likelihoods}
		
			As we only observe from $g$ and never from $f$ directly, there is no need to add noise from $f(x)$ to $z$ and we degenerate the likelihood to $z = f(x)$. The likelihood for $g$ is the regular Gaussian likelihood due to observational noise. Together, we have
			\begin{equation}
			\begin{aligned}
				p(z | \bvec{v}) &= \mathcal{N}(z ; \bvec{v}^{T} \psi(y), \sigma^{2}), \\
				p(z | \bvec{w}) &= \mathcal{N}(z ; \bvec{w}^{T} \phi(x), 0).
			\end{aligned}
			\end{equation}
	
		\paragraph{Prior for $g$}
		
			The prior on the weights of $g$ is
			\begin{equation}
				p(\bvec{v}) = \mathcal{N}(\bvec{v} ; \bvec{0}, \beta^{2} I).
			\end{equation}
		
		\paragraph{Likelihood for $g$}
		
			In task transformed learning, the pairs $(\bvec{y}, \bvec{z})$ are used to learn $g$, and $(\tilde{\bvec{y}}, \tilde{\bvec{z}})$ are the query points for $g$. Although $\bvec{z}$ is not directly available, they are propagated through from $f$. We also refer to $\bvec{z}$ as the pseudo-training targets. This leads to the following likelihood,
			\begin{equation}
			\begin{aligned}
				p(\bvec{z} | \bvec{v}) &= \mathcal{N}(\bvec{z} ; \Psi^{T} \bvec{v}, \sigma^{2} I), \\
				p(\tilde{\bvec{z}} | \bvec{v}) &= \mathcal{N}(\tilde{\bvec{z}} ; \tilde{\Psi}^{T} \bvec{v}, \sigma^{2} I).
			\end{aligned}
			\end{equation}
		
		\paragraph{Marginal Likelihood for $g$}
		
			The marginal likelihood of observing the pseudo-training targets $\bvec{z}$ is
			\begin{equation}
			\begin{aligned}
				p(\bvec{z}) &= \int_{\mathbb{R}^{q}} p(\bvec{z} | \bvec{v}) p(\bvec{v}) d\bvec{v} \\
				&= \mathcal{N}(\bvec{z} ; \bvec{0}, \beta^{2} \Psi^{T} \Psi + \sigma^{2} I).		
			\end{aligned}
			\end{equation}
		
		\paragraph{Posterior for $g$}
		
			The posterior of the weights given the pseudo-training targets $\bvec{z}$ is
			\begin{equation}
			\begin{aligned}
				p(\bvec{v} | \bvec{z}) &= \frac{p(\bvec{z} | \bvec{v}) p(\bvec{v})}{p(\bvec{z})} \\
				&= \mathcal{N}\Bigg(\bvec{v} ; \bigg(\Psi \Psi^{T} + \frac{\sigma^{2}}{\beta^{2}} I\bigg)^{-1} \Psi \bvec{z}, \sigma^{2} \bigg( \Psi \Psi^{T} + \frac{\sigma^{2}}{\beta^{2}} I \bigg)^{-1}\Bigg).
			\end{aligned} 
			\end{equation}
		
		\paragraph{Predictive distribution for $g$}
		
			The posterior predictive distribution of $\tilde{\bvec{z}}$ given the pseudo-training targets $\bvec{z}$ is
			\begin{equation}
			\begin{aligned}
				p(\tilde{\bvec{z}} | \bvec{z}) &= \int_{\mathbb{R}^{q}} p(\tilde{\bvec{z}} | \bvec{v}) p(\bvec{v} | \bvec{z}) d\bvec{v} \\
				&= \mathcal{N}\Bigg(\tilde{\bvec{z}} ; \tilde{\Psi}^{T} \bigg(\Psi \Psi^{T} + \frac{\sigma^{2}}{\beta^{2}} I \bigg)^{-1} \Psi \bvec{z}, \sigma^{2} \tilde{\Psi}^{T} \bigg( \Psi \Psi^{T} + \frac{\sigma^{2}}{\beta^{2}} I \bigg)^{-1} \tilde{\Psi} + \sigma^{2} I \Bigg) \\
				&= \mathcal{N}(\tilde{\bvec{z}} ; A^{T} \bvec{z}, \Sigma),
			\end{aligned}
			\end{equation}
			where $A = \Psi^{T} (\Psi \Psi^{T} + \frac{\sigma^{2}}{\beta^{2}} I )^{-1} \tilde{\Psi}$ and $\Sigma = \sigma^{2} \tilde{\Psi}^{T} ( \Psi \Psi^{T} + \frac{\sigma^{2}}{\beta^{2}} I )^{-1} \tilde{\Psi} + \sigma^{2} I$.
		
			Importantly, the \gls{MAP} solution for learning $g$ amount to just taking the posterior mean $\hat{\bvec{v}} = (\Psi \Psi^{T} + \frac{\sigma^{2}}{\beta^{2}} I )^{-1} \Psi \bvec{z}$ as a point estimate. In this case, the predictive covariance would simplify to $\Sigma = \sigma^{2} I$.
			
		\paragraph{Prior for $f$}
		
			The prior on the weights of $f$ is
			\begin{equation}
				p(\bvec{w}) = \mathcal{N}(\bvec{w} ; \bvec{0}, \gamma^{2} I).
			\end{equation}

		\paragraph{Likelihood for $f$}
		
			As targets $z$ are never directly observed from $f$, the likelihood is a noiseless Gaussian likelihood,
			\begin{equation}
			\begin{aligned}
				p(\bvec{z} | \bvec{w}) &= \mathcal{N}(\bvec{z} ; \Phi^{T} \bvec{w}, 0 I), \\
				p(\bvec{z}^{\star} | \bvec{w}) &= \mathcal{N}(\bvec{z}^{\star} ; (\Phi^{\star})^{T} \bvec{w}, 0 I).
			\end{aligned}
			\end{equation}
			Propagating this likelihood through the predictive distribution of $g$, we have
			\begin{equation}
			\begin{aligned}
				p(\tilde{\bvec{z}} | \bvec{w}) &= \int_{\mathbb{R}^{n}} p(\tilde{\bvec{z}} | \bvec{z}) p(\bvec{z} | \bvec{w}) d\bvec{z} \\
				&= \mathcal{N}(\tilde{\bvec{z}} ; A^{T} \Phi^{T} \bvec{w}, \Sigma).
			\end{aligned}
			\end{equation}
			The above prior-likelihood pair describes a \gls{TBLR} with $M = A = \Psi^{T} (\Psi \Psi^{T} + \frac{\sigma^{2}}{\beta^{2}} I )^{-1} \tilde{\Psi}$ as the transformation matrix and $\Sigma = \sigma^{2} \tilde{\Psi}^{T} ( \Psi \Psi^{T} + \frac{\sigma^{2}}{\beta^{2}} I )^{-1} \tilde{\Psi} + \sigma^{2} I$ as the noise covariance. As such, the remaining distributions exhibit the same forms as shown in \cref{tab:transformed_regression}.
			
		\paragraph{Marginal Likelihood for $f$}
		
			The marginal likelihood for the observed targets $\tilde{\bvec{z}}$ is
			\begin{equation}
			\begin{aligned}
				p(\tilde{\bvec{z}}) &= \int_{\mathbb{R}^{p}} p(\tilde{\bvec{z}} | \bvec{w}) p(\bvec{w}) d\bvec{w} \\
				&= \mathcal{N}(\tilde{\bvec{z}} ; \bvec{0}, \gamma^{2}  A^{T} \Phi^{T} \Phi A + \Sigma) \\
				&= \mathcal{N}(\tilde{\bvec{z}} ; \bvec{0}, [\Sigma^{-1} - \Sigma^{-1} A^{T} \Phi^{T} C \Phi A \Sigma^{-1}]^{-1}),
			\end{aligned}
			\end{equation}
			where $C = [\Phi A \Sigma^{-1} A^{T} \Phi^{T} + \frac{1}{\gamma^{2}} I]^{-1}$. The last line is an alternative form that is more computationally efficient when the number of features is less than $p < m$ where $p$ is the dimensionality of the feature $\phi(x)$ for $f$.
			
		\paragraph{Posterior for $f$}
		
			The posterior of the weights $\bvec{w}$ given the observed targets $\tilde{\bvec{z}}$ is
			\begin{equation}
			\begin{aligned}
				p(\bvec{w} | \tilde{\bvec{z}}) &= \frac{p(\tilde{\bvec{z}} | \bvec{w}) p(\bvec{w})}{p(\tilde{\bvec{z}})} \\
				&= \mathcal{N}(\bvec{w} ;  \bvec{m}, C),
			\end{aligned}
			\end{equation}
			where $\bvec{m} := C \Phi A \Sigma^{-1} \tilde{\bvec{z}}$.
		
		\paragraph{Predictive distribution for $f$}
		
			Finally, the overall predictive distribution of query targets $\bvec{z}^{\star}$ given the observed targets $\tilde{\bvec{z}}$ is
			\begin{equation}
			\begin{aligned}
				p(\bvec{z}^{\star} | \tilde{\bvec{z}}) &= \int_{\mathbb{R}^{p}} p(\bvec{z}^{\star} | \bvec{w}) p(\bvec{w} | \tilde{\bvec{z}}) d\bvec{w} \\
				&= \mathcal{N}(\bvec{z}^{\star} ; {\Phi^{\star}}^{T} \bvec{m}, {\Phi^{\star}}^{T} C \Phi^{\star}).
			\end{aligned}
			\end{equation}
			
		Consider the posterior mean $\bvec{m} := C \Phi A \Sigma^{-1} \tilde{\bvec{z}} = [\Phi A \Sigma^{-1} A^{T} \Phi^{T} + \frac{1}{\gamma^{2}} I]^{-1} \Phi M \Sigma^{-1} \tilde{\bvec{z}}$, which would also be the \gls{MAP} solution for $f$. Using the \gls{MAP} solution for learning $g$ such that $\Sigma = \sigma^{2} I$, we have $\bvec{m} := [\Phi A A^{T} \Phi^{T} + \frac{\sigma^{2}}{\gamma^{2}} I]^{-1} \Phi A \tilde{\bvec{z}}$. This is the same form as the weights $\tilde{\bvec{w}}$ of the parametric \gls{DME} estimator (\cref{def:parametric_dme_estimator}) with $\lambda = \frac{\sigma^{2}}{n \beta^{2}}$ and $\epsilon = \frac{\sigma^{2}}{m \gamma^{2}}$. 
	\end{proof}	

		\setlength{\tabcolsep}{0pt}
		\begin{table*}[!]
			\caption{Summary of TBLR and TBKR, where $C := [\Phi M \Sigma^{-1} M^{T} \Phi^{T} + \frac{1}{\gamma^{2}} I]^{-1}$, $\bvec{m} := C \Phi M \Sigma^{-1} \tilde{\bvec{z}}$, and $S := M^{T} K M + \Sigma$.}
			\label{tab:transformed_regression}
			\centering
			\begin{tabular}{lcc}
				Density & Transformed Bayesian Linear Regression & Transformed Bayesian Kernel Regression \\
				\midrule
				Prior & $p(\bvec{w}) = \mathcal{N}(\bvec{w}; \bvec{0}, \gamma^{2})$ & $p(\bvec{f}) = \mathcal{N}(\bvec{f} ; \bvec{0}, K)$ \\
				Likelihood & $p(\tilde{\bvec{z}} | \bvec{w}) = \mathcal{N}(\tilde{\bvec{z}} ; M^{T} \Phi^{T} \bvec{w}, \Sigma)$ & $p(\tilde{\bvec{z}} | \bvec{f}) = \mathcal{N}(\tilde{\bvec{z}} ; M^{T} \bvec{f}, \Sigma)$ \\
				Evidence & $ p(\tilde{\bvec{z}}) = \mathcal{N}(\tilde{\bvec{z}} ; \bvec{0}, [\Sigma^{-1} - \Sigma^{-1} M^{T} \Phi^{T} C \Phi M \Sigma^{-1}]^{-1}) $ & $p(\tilde{\bvec{z}}) = \mathcal{N}(\tilde{\bvec{z}} ; \bvec{0}, M^{T} K M + \Sigma)$ \\
				Posterior & $p(\bvec{w} | \tilde{\bvec{z}}) = \mathcal{N}(\bvec{w} ;  \bvec{m}, C)$ & $p(\bvec{f} | \tilde{\bvec{z}}) = \mathcal{N}(\bvec{f} ; K M S^{-1} \tilde{\bvec{z}}, K - K M S^{-1} M^{T} K)$ \\
				Predictive & $p(\bvec{z}^{\star} | \tilde{\bvec{z}}) = \mathcal{N}(\bvec{z}^{\star} ; {\Phi^{\star}}^{T} \bvec{m}, {\Phi^{\star}}^{T} C \Phi^{\star})$ & $p(\bvec{z}^{\star} | \tilde{\bvec{z}}) = \mathcal{N}(\bvec{z}^{\star} ;  {K^{\star}}^{T} M S^{-1} \tilde{\bvec{z}}, K^{\star \star} - {K^{\star}}^{T} M S^{-1} M^{T} K^{\star})$
			\end{tabular}
		\end{table*}
		
	\begin{proof}[\textbf{Proof of \Cref{thm:ttblr_ttbkr} Part 2} -- \Acrfull{TTBKR}]
		In this proof we provide the derivations for \acrfull{TTBKR}, also named \acrfull{TTGPR}, whose graphical model is provided in \cref{fig:ttgpr}. We first reiterate the priors and likelihoods used.
		
		\paragraph{Priors}
			
			We place \gls{GP} priors on the functions $g$ and $f$ directly,
			\begin{equation}
			\begin{aligned}
				g &\sim \mathcal{GP}(0, \ell), \\
				f &\sim \mathcal{GP}(0, k).
			\end{aligned}
			\end{equation}
		
		\paragraph{Likelihoods} 
		
			As we only observe from $g$ and never from $f$ directly, there is no need to add noise from $f(x)$ to $z$ and we degenerate the likelihood to $z = f(x)$. The likelihood for $g$ is the regular Gaussian likelihood due to observational noise. Together, we have
			\begin{equation}
			\begin{aligned}
				p(z | g) &= \mathcal{N}(z ; g(y), \sigma^{2}), \\
				p(z | f) &= \mathcal{N}(z ; f(x), 0).
			\end{aligned}
			\end{equation}
		
		\paragraph{Prior for $g$}
		
			The prior of $g$ at $\bvec{y}$ is
			\begin{equation}
				p(\bvec{g}) = \mathcal{N}(\bvec{g} ; \bvec{0}, L).
			\end{equation}
		
		\paragraph{Likelihood for $g$}
		
			In task transformed learning, the pairs $(\bvec{y}, \bvec{z})$ are used to learn $g$, and $(\tilde{\bvec{y}}, \tilde{\bvec{z}})$ are the query points for $g$. Although $\bvec{z}$ is not directly available, they are propagated through from $f$. We also refer to $\bvec{z}$ as the pseudo-training targets. The likelihood of observing $\bvec{z}$ at $\bvec{y}$ is,
			\begin{equation}
				p(\bvec{z} | \bvec{g}) = \mathcal{N}(\bvec{z} ; \bvec{g}, \sigma^{2} I).
			\end{equation}
		
		\paragraph{Marginal Likelihood for $g$}
		
			The marginal likelihood of observing the psuedo-training targets $\bvec{z}$ is:
			\begin{equation}
			\begin{aligned}
				p(\bvec{z}) &= \int_{\mathbb{R}^{n}} p(\bvec{z} | \bvec{g}) p(\bvec{g}) d\bvec{g} \\ &= \mathcal{N}(\bvec{z} ; \bvec{0}, L + \sigma^{2} I).
			\end{aligned}
			\end{equation}
		
		\paragraph{Posterior for $g$}
		
			The posterior of the latent function evaluations $\bvec{g}$ at $\bvec{y}$ given the pseudo-training targets $\bvec{z}$ is
			\begin{equation}
			\begin{aligned}
				p(\bvec{g} | \bvec{z}) &= \frac{p(\bvec{z} | \bvec{g}) p(\bvec{g})}{p(\bvec{z})} \\
				&= \mathcal{N}(\bvec{g} ; L (L + \sigma^{2} I)^{-1} \bvec{z}, L - L (L + \sigma^{2} I)^{-1} L).
			\end{aligned}
			\end{equation}
		
		\paragraph{Predictive distribution for $g$}
		
			To obtain the predictive distribution, we first condition the \gls{GP} field for on the latent function evaluations $\bvec{g}$ at $\bvec{y}$ and to obtain the conditional distribution for $\tilde{\bvec{g}}$ at $\tilde{\bvec{y}}$ given $\bvec{g}$ at $\bvec{y}$,
			\begin{equation}
				p(\tilde{\bvec{g}} | \bvec{g}) = \mathcal{N}(\tilde{\bvec{g}}; \tilde{L}^{T} L^{-1} \bvec{g}, \tilde{\tilde{L}} - \tilde{L}^{T} L^{-1} \tilde{L}).
			\end{equation}
			where $\tilde{\tilde{L}} := \tilde{\Psi}^{T} \tilde{\Psi}$. Now, marginalize the conditional field against the posterior,
			\begin{equation}
			\begin{aligned}
				p(\tilde{\bvec{g}} | \bvec{z}) &= \int_{\mathbb{R}^{n}} p(\tilde{\bvec{g}} | \bvec{g}) p(\bvec{g} | \bvec{z}) d\bvec{g} \\
				&= \mathcal{N}(\tilde{\bvec{g}}; \tilde{L}^{T} (L + \sigma^{2} I)^{-1} \bvec{z}, \tilde{\tilde{L}} - \tilde{L}^{T} (L + \sigma^{2} I)^{-1} \tilde{L})
			\end{aligned}
			\end{equation}
			Finally, marginalize the likelihood $p(\tilde{\bvec{g}} | \bvec{z})$ with the predictive distribution of the latent evaluations $\tilde{\bvec{g}}$ to get the final predictive distribution of the observations $\tilde{\bvec{z}}$,
			\begin{equation}
			\begin{aligned}
				p(\tilde{\bvec{z}} | \bvec{z}) &= \int_{\mathbb{R}^{m}} p(\tilde{\bvec{z}} | \tilde{\bvec{g}}) p(\tilde{\bvec{g}} | \bvec{z}) d\tilde{\bvec{g}} \\
				&= \mathcal{N}(\tilde{\bvec{z}} ; \tilde{L}^{T} (L + \sigma^{2} I)^{-1} \bvec{z}, \tilde{\tilde{L}} + \sigma^{2} I - \tilde{L}^{T} (L + \sigma^{2} I)^{-1} \tilde{L}) \\
				&= \mathcal{N}(\tilde{\bvec{z}} ; A^{T} \bvec{z}, \Sigma),
			\end{aligned}
			\end{equation}
			where $A = (L + \sigma^{2} I)^{-1} \tilde{L}$ and $\Sigma = \tilde{\tilde{L}} + \sigma^{2} I - \tilde{L}^{T} (L + \sigma^{2} I)^{-1} \tilde{L}$.
			
			Importantly, the \gls{MAP} solution for learning $g$ amount to just taking the posterior mean $\tilde{\bvec{g}} = \tilde{L}^{T} (L + \sigma^{2} I)^{-1} \bvec{z}$ as a point estimate. In this case, the predictive covariance would simplify to $\Sigma = \sigma^{2} I$.
			
		\paragraph{Prior for $f$}
		
			The prior of $f$ at $\bvec{x}$ is 	
			\begin{equation}
				p(\bvec{f}) = \mathcal{N}(\bvec{f} ; \bvec{0}, K).
			\end{equation}
		
		\paragraph{Likelihood for $f$}
			
			As targets $z$ are never directly observed from $f$, the likelihood is a noiseless Gaussian likelihood,
			\begin{equation}
				p(\bvec{z} | \bvec{f}) = \mathcal{N}(\bvec{z} ; \bvec{f}, 0 I).
			\end{equation}
			Propagating this likelihood through the predictive distribution of $g$, we have
			\begin{equation}
			\begin{aligned}
				p(\tilde{\bvec{z}} | \bvec{f}) &= \int_{\mathbb{R}^{n}} p(\tilde{\bvec{z}} | \bvec{z}) p(\bvec{z} | \bvec{f}) d \bvec{z} \\
				&= \mathcal{N}(\tilde{\bvec{z}} ; A^{T} \bvec{f}, \Sigma).
			\end{aligned}
			\end{equation}
		
			The above prior-likelihood pair describes a \gls{TBKR} with $M = A = (L + \sigma^{2} I)^{-1} \tilde{L}$ as the transformation matrix and $\Sigma = \tilde{\tilde{L}} + \sigma^{2} I - \tilde{L}^{T} (L + \sigma^{2} I)^{-1} \tilde{L}$ as the noise covariance. As such, the remaining distribution exhibit the same forms as shown in \cref{tab:transformed_regression}.
			
		\paragraph{Marginal Likelihood for $f$}

			The marginal likelihood for the observed targets $\tilde{\bvec{z}}$ is 
			\begin{equation}
			\begin{aligned}
				p(\tilde{\bvec{z}}) &= \int_{\mathbb{R}^{n}} p(\tilde{\bvec{z}} | \bvec{f}) p(\bvec{f}) d\bvec{f} \\
				&= \mathcal{N}(\tilde{\bvec{z}} ; \bvec{0}, A^{T} K A + \Sigma).
			\end{aligned}
			\end{equation}
		
		\paragraph{Posterior for $f$}
			
			The posterior of the function evaluations $\bvec{f}$ at $\bvec{x}$ given the observed targets $\tilde{\bvec{z}}$ is
			\begin{equation}
			\begin{aligned}
				p(\bvec{f} | \tilde{\bvec{z}}) &= \frac{p(\tilde{\bvec{z}} | \bvec{f}) p(\bvec{f})}{p(\tilde{\bvec{z}})} \\
				&= \mathcal{N}(\bvec{f} ; K A (A^{T} K A + \Sigma)^{-1} \tilde{\bvec{z}}, K - K A (A^{T} K A + \Sigma)^{-1} A^{T} K).
			\end{aligned}
			\end{equation}
		
		\paragraph{Predictive distribution for $f$}
		
			Finally, to obtain the predictive distribution we first condition the \gls{GP} field on the latent function evaluations $\bvec{f}$ to obtain the conditional distribution for $\bvec{f}^{\star}$ at $\bvec{x}^{\star}$ given $\bvec{f}$ at $\bvec{x}$.
			\begin{equation}
				p(\bvec{f}^{\star} | \bvec{f}) = \mathcal{N}(\bvec{f}^{\star} ; (K^{\star})^{T} K^{-1} \bvec{f}, K^{\star \star} - (K^{\star})^{T} K^{-1} K^{\star} ).
			\end{equation}
			Now, marginalize the conditional field against the posterior,
			\begin{equation}
			\begin{aligned}
				p(\bvec{f}^{\star} | \tilde{\bvec{z}}) &= \int_{\mathbb{R}^{n}} p(\bvec{f}^{\star} | \bvec{f}) p(\bvec{f} | \tilde{\bvec{z}}) d\bvec{f} \\
				&= \mathcal{N}(\bvec{f}^{\star} ; (K^{\star})^{T} A (A^{T} K A + \Sigma)^{-1} \tilde{\bvec{z}}, K^{\star \star} - (K^{\star})^{T} A (A^{T} K A + \Sigma)^{-1} A^{T} K^{\star}).
			\end{aligned}
			\end{equation}
			Finally, the overall predictive distribution of query targets $\bvec{z}^{\star}$ given the observed targets $\tilde{\bvec{z}}$ is
			\begin{equation}
			\begin{aligned}
				p(\bvec{z}^{\star} | \tilde{\bvec{z}}) &= \int_{\mathbb{R}^{n}} p(\bvec{z}^{\star} | \bvec{f}^{\star}) p(\bvec{f}^{\star} | \tilde{\bvec{z}}) d\bvec{f}^{\star} \\
				&= \mathcal{N}(\bvec{z}^{\star} ; (K^{\star})^{T} A (A^{T} K A + \Sigma)^{-1} \tilde{\bvec{z}}, K^{\star \star} - (K^{\star})^{T} A (A^{T} K A + \Sigma)^{-1} A^{T} K^{\star}).
			\end{aligned}
			\end{equation}
			
			Consider the posterior predictive mean at a particular query point $x$, $\bar{f}(x) = (\bvec{k}(x))^{T} A (A^{T} K A + \Sigma)^{-1} \tilde{\bvec{z}} = \tilde{\bvec{z}}^{T} (A^{T} K A + \Sigma)^{-1} A^{T} \bvec{k}(x)$. Using the \gls{MAP} solution for learning $g$ such that $\Sigma = \sigma^{2} I$, we have $\bar{f}(x) = \tilde{\bvec{z}}^{T} (A^{T} K A + \sigma^{2} I)^{-1} A^{T} \bvec{k}(x)$. This is the same form as the nonparametric \gls{DME} estimator (\cref{def:parametric_dme_estimator}) with $\lambda = \frac{\sigma^{2}}{n}$ and $\epsilon = \frac{\sigma^{2}}{m}$. 
	\end{proof}
		
\newpage
\section{Supporting Proofs for \Cref{sec:nonparametric_bayes_rule}}
\label{sec:nonparametric_bayes_rule_proofs}

	\begin{proof}[\textbf{Proof of \Cref{thm:dmo_for_nbr}}]
		We first factorize the joint operator $C_{YX} = (C_{XY})^{T}$ in both directions,
		\begin{equation}
			C_{Y|X} C_{XX} = C_{YX} = (C_{XY})^{T} = (C_{X|Y} C_{YY})^{T} = C_{YY} (C_{X|Y})^{T}.
		\end{equation}
		This is analogous to the equation $p(y|x) p(x) = p(y, x) = p(x, y) = p(x | y) p(y) = p(y) p(x | y)$.
		
		Since $C_{X|Y} C_{Y|X} C_{XX} = C_{XX}$, we then apply $C_{X|Y}$ on both sides to cancel out $C_{Y|X}$ and obtain the equation for $C_{XX}$,
		\begin{equation}
			C_{XX} = (C_{X|Y} C_{Y|X}) C_{XX} = C_{X|Y} (C_{Y|X} C_{XX}) = C_{X|Y} C_{YY} (C_{X|Y})^{T}.
		\end{equation}
		This is analogous to the equation $p(x) = \int_{\mathcal{Y}} p(y | x) dy p(x) = \int_{\mathcal{Y}} p(y | x) p(x) dy = \int_{\mathcal{Y}} p(y) p(x | y) dy$.
		
		Hence, 
		\begin{equation}
			C_{XX}' := C_{X|Y} C_{YY} (C_{X|Y})^{T} = C_{XX}.
		\end{equation}
		Finally, from \cref{thm:dmo} we have
		\begin{equation}
			C_{X|Y}' = (C_{X|Y} C_{YY})^{T} ( C_{X|Y} C_{YY} (C_{X|Y})^{T})^{-1} = C_{YX} C_{XX}^{-1} = C_{Y|X}.
		\end{equation}
	\end{proof}
	
	\begin{proof}[\textbf{Proof of \Cref{thm:emp_dmo_special_case}}]
		Since $m = n$ and $\tilde{\bvec{y}} = \bvec{y}$, we have that $\tilde{L} = L$, $\tilde{\Psi} = \Psi$. Consequently, $\lim_{\lambda \to 0^{+}} A = \lim_{\lambda \to 0^{+}} (L + n \lambda I)^{-1} L = I$. Substituting this into \eqref{eq:emp_cmo} we have
		\begin{equation}
		\begin{aligned}
			\lim_{\lambda \to 0^{+}} \bar{C}_{X|Y}' &= \lim_{\lambda \to 0^{+}} \tilde{\Psi} \big[ A^{T} K A + m \epsilon I \big]^{-1} A^{T} \Phi^{T}
			&= \Psi \big[ I^{T} K I + n \epsilon I \big]^{-1} I^{T} \Phi^{T}
			&= \Psi \big[K + n \epsilon I \big]^{-1} \Phi^{T}.
		\end{aligned}
		\end{equation}
		Reversing the roles of $X$ and $Y$ in \eqref{eq:emp_cmo} and replacing the notation $\lambda$ with $\epsilon$, we have that $\hat{C}_{Y|X} = \Psi \big[K + n \epsilon I \big]^{-1} \Phi^{T}$. This concludes the proof.
	\end{proof}
	
\newpage
\section{Parallels between Probabilistic Rules and Mean Operators for \Cref{sec:nonparametric_bayes_rule}}
\label{sec:dmo_for_nbr}

	Both the usual and nonparametric Bayes' rule are derived to reverse the relationship specified by the likelihood (density or operator, resp.) by matching the joint. In both cases, the prior (density or operator, resp.) is inevitably required to perform this computation.
		
	Consider the derivation for Bayes' rule. When given a \textit{forward} density $p(x | y)$ and a marginal density on its conditioned variable $p(y)$ which specifies a joint $p(x , y) = p(x | y) p(y)$, we seek a \textit{backward} density $q(y | x)$ and a marginal density $q(x)$ that would yield the same joint $q(y | x) q(x) = p(x, y) = p(x | y) p(y)$. It is only when applying $\int_{\mathcal{Y}} \cdot dy$ on both sides, requiring that $q(y | x)$ is a density, that we have $q(x) = \int_{\mathcal{Y}} p(x | y) p(y) dy$ and thus Bayes' rule.
		
	Similarly, when given a \textit{forward} \gls{CMO} $C_{X|Y} : \mathcal{H}_{\ell} \to \mathcal{H}_{k}$ and a symmetric operator $C_{YY} : \mathcal{H}_{\ell} \to \mathcal{H}_{\ell}$ on its conditioned variable which specifies a joint $C_{XY} = C_{X|Y} C_{YY}$, we seek a \textit{backward} operator $D_{Y|X} : \mathcal{H}_{k} \to \mathcal{H}_{\ell}$ and a symmetric operator $D_{XX} : \mathcal{H}_{k} \to \mathcal{H}_{k}$ that would yield the same joint $D_{Y|X} D_{XX} = C_{YX} = (C_{XY})^{T} = (C_{X|Y} C_{YY})^{T}$. Without further requirement we see that $D_{Y|X} = C_{X|Y}'$ \eqref{eq:dmo} and $D_{XX} = C_{XX}'$ is one solution. It is only when applying $C_{X|Y}$ on both sides, requiring the assumption of \cref{thm:dmo_for_nbr}, that we have $D_{Y|X} = C_{Y|X}$ and $D_{XX} = C_{XX}$ and thus a nonparametric Bayes' rule.
	
	Importantly, it is only when \glspl{DMO} and \gls{KBR} are viewed as a statement for relationship between $X$ and $Y$ that they are seen as nonparametric versions of the Bayes' rule. However, \glspl{DMO} and \gls{KBR} are not Bayesian models with respect to the task of inferring deconditional mean or conditional means. This is because both models only infer point estimates for the deconditional or conditional mean, and no measure of uncertainty in the inferred function is provided. 

	\Cref{tab:encoded_expectations} review mean embeddings and their encoded expectations, providing probabilistic interpretations to \gls{RKHS} embeddings and operators \citep{song2013kernel}.
		
	\setlength{\tabcolsep}{5pt}
	\begin{table}
	\caption{Mean embeddings and their encoded expectations. Switch $X \leftrightarrow Y$ for all combinations. Since the bottom two rows do not apply for the first column, additional equivalences for the last column are provided instead. The kernels $k : \mathcal{X} \times \mathcal{X} \to \mathbb{R}$ and $\ell : \mathcal{Y} \times \mathcal{Y} \to \mathbb{R}$ are positive definite and characteristic on their respective spaces $\mathcal{X}$ and $\mathcal{Y}$. They define the \gls{RKHS} $\mathcal{H}_{k}$ and $\mathcal{H}_{\ell}$ respectively. We define $\mathcal{H}_{\ell \ell} := \mathcal{H}_{\ell} \otimes \mathcal{H}_{\ell}$ and $\mathcal{H}_{k \ell} := \mathcal{H}_{k} \otimes \mathcal{H}_{\ell}$ and let $g, g' \in \mathcal{H}_{\ell}$ and $f \in \mathcal{H}_{k}$ be generic example functions within each \gls{RKHS}.}
	\label{tab:encoded_expectations}
	\centering
	\begin{tabular}{lcccc}
		Random & $Y$ & $(Y, Y)$ & $(X, Y)$ & $X | Y = y$ \\
		Variable & $:\Omega \to \mathcal{Y}$ & $:\Omega \to \mathcal{Y} \times \mathcal{Y}$ & $:\Omega \to \mathcal{X} \times \mathcal{Y}$ & $:\Omega \to \mathcal{X}$ \\
		\midrule
		Density & $p_{Y}$ \hspace*{\fill} $\in \mathcal{P}_{\mathcal{Y}}$ & $p_{YY}$ \hspace*{\fill} $\in \mathcal{P}_{\mathcal{Y} \times \mathcal{Y}}$ & $p_{XY}$ \hspace*{\fill} $\in \mathcal{P}_{\mathcal{X} \times \mathcal{Y}}$ & $p_{X | Y = y}$ \hspace*{\fill} $\in \mathcal{P}_{\mathcal{X}}$ \\
		Function & $p_{Y}(y)$ \hspace*{\fill} $\in \mathbb{R}^{+}$ & $p_{YY}(y, y')$ \hspace*{\fill} $\in \mathbb{R}^{+}$ & $p_{XY}(x, y)$ \hspace*{\fill} $\in \mathbb{R}^{+}$ &  $p_{X | Y = y}(x)$ \hspace*{\fill} $\in \mathbb{R}^{+}$ \\
		\midrule
		Mean Map & $\mu_{Y} :=$ & $\mu_{YY} :=$ & $\mu_{XY} :=$ & $\mu_{X | Y = y} = $ \\
		Definition & $\mathbb{E}[\ell(Y, \cdot)]$ & $\mathbb{E}[\ell(Y, \cdot) \ell(Y, \cdot)^{T}]$ & $\mathbb{E}[k(X, \cdot) \ell(Y, \cdot)^{T}]$ & $\mathbb{E}[k(X, \cdot) | Y = y] $ \\
		\midrule
		Mean & $\mu_{Y}$ \hspace*{\fill} $\in \mathcal{H}_{\ell}$ & $\mu_{YY}$ \hspace*{\fill} $\in \mathcal{H}_{\ell \ell}$ & $\mu_{XY}$ \hspace*{\fill} $\in \mathcal{H}_{k \ell}$ & $\mu_{X | Y = Y}$ \hspace*{\fill} $\in \mathcal{H}_{k}$ \\
		Embedding & $\mu_{Y}(y)$ \hspace*{\fill} $\in \mathbb{R}$ & $\mu_{YY}(y, y')$ \hspace*{\fill} $\in \mathbb{R}$ & $\mu_{XY}(x, y)$ \hspace*{\fill} $\in \mathbb{R}$ & $\mu_{X | Y = y}(x)$ \hspace*{\fill} $\in \mathbb{R}$ \\
		\midrule
		Encoded & $\langle \mu_{Y}, g \rangle_{\mathcal{H}_{\ell}}$ & $\langle \mu_{YY}, g' g^{T} \rangle_{\mathcal{H}_{\ell \ell}}$ & $\langle \mu_{XY}, f g^{T} \rangle_{\mathcal{H}_{k \ell}}$ & $\langle \mu_{X | Y = y}, f \rangle_{\mathcal{H}_{k}}$ \\
		Expectation & $= \mathbb{E}[g(Y)]$ & $= \mathbb{E}[g'(Y) g(Y)]$ & $= \mathbb{E}[f(X) g(Y)]$ & $= \mathbb{E}[f(X) | Y = y] $ \\
		\midrule
		Operator & $C_{X|Y} C_{YY}$ & $C_{YY} := \mu_{YY}$ & $C_{XY} := \mu_{XY}$ & $C_{X|Y} \ell(y, \cdot) := $ \\
		Definition & $= C_{XY}$ & $(C_{YY})^{T} = C_{YY}$ & $(C_{XY})^{T} = C_{YX}$ & $\mu_{X | Y = y}$ \\
		\midrule
		Encoded & $f^{T} C_{XY} =$ & $\langle g', C_{YY} g \rangle_{\mathcal{H}_{\ell}}$ & $\langle f, C_{XY} g \rangle_{\mathcal{H}_{k}}$ & $(C_{X|Y})^{T} f = g := $ \\
		Expectation & $g^{T} C_{YY}$ & $= \mathbb{E}[g'(Y) g(Y)]$ & $= \mathbb{E}[f(X) g(Y)]$ & $ \mathbb{E}[f(X) | Y = \cdot]$ \\
	\end{tabular}
	\end{table}
	
\newpage	
\section{Connections between the Deconditional Mean Operator and Kernel Bayes' Rule for \Cref{sec:related_work}}
\label{sec:connections_to_kbr}

	\setlength{\tabcolsep}{4pt}
	\begin{table*}[!t]
		\caption{Empirical estimators for \gls{DMO} and \gls{KBR}. We use the shorthand $A := (L + n \lambda I)^{-1} \tilde{L}$ and $D := \mathrm{diag}(A \bvec{1})$.}
		\label{tab:dmo_kbr_comparison}
		\centering
		\begin{tabular}{lcccc}
			Method & Joint Operator & Evidence Operator & Posterior Operator & Computational Form \\
			& $\bar{C}_{XY}$ & $\bar{C}_{XX}$ or $\bar{C}_{XX}'$ & $\bar{C}_{Y|X}$ or $\bar{C}_{X|Y}'$ & $\bar{C}_{Y|X}$ or $\bar{C}_{X|Y}'$ \\
			\midrule
			DMO &  $\hat{C}_{X|Y} \tilde{C}_{YY}$ & $\hat{C}_{X|Y} \tilde{C}_{YY} (\hat{C}_{X|Y})^{T}$ & $(\bar{C}_{XY})^{T} (\bar{C}_{XX}' + \epsilon I)^{-1}$ & $\tilde{\Psi} \big[ A^{T} K A + m \epsilon I \big]^{-1} A^{T} \Phi^{T}$ \\
			DMO(W) &  $\hat{C}_{X|Y} \tilde{C}_{YY}$ & $\hat{C}_{X|Y} \tilde{C}_{YY} (\hat{C}_{X|Y})^{T}$ & $(\bar{C}_{XY})^{T} (\bar{C}_{XX}' + \epsilon I)^{-1}$ & $\tilde{\Psi} A^{T} \big[ K A A^{T} + m \epsilon I \big]^{-1} \Phi^{T}$ \\
			KBR(a)-I & $\hat{C}_{X|Y} \tilde{C}_{YY}$ & $\hat{C}_{XX|Y} \tilde{\mu}_{Y}$ & $(\bar{C}_{XY})^{T} (\bar{C}_{XX} + \epsilon I)^{-1}$ & $\tilde{\Psi} A^{T} \big[ KD + m \epsilon I \big]^{-1} \Phi^{T}$ \\
			KBR(a)-II & $\hat{C}_{X|Y} \tilde{C}_{YY}$ & $\hat{C}_{XX|Y} \tilde{\mu}_{Y}$ & $(\bar{C}_{XY})^{T} (\bar{C}_{XX}^{2} + \epsilon I)^{-1} \bar{C}_{XX}$ & $\tilde{\Psi} A^{T} \big[ (KD)^{2} + m^{2} \epsilon I \big]^{-1} K D \Phi^{T}$ \\
			KBR(b)-I & $\hat{C}_{XY|Y} \tilde{\mu}_{Y}$ & $\hat{C}_{XX|Y} \tilde{\mu}_{Y}$ & $(\bar{C}_{XY})^{T} (\bar{C}_{XX} + \epsilon I)^{-1}$ & $\Psi D \big[ KD + m \epsilon I \big]^{-1} \Phi^{T}$ \\
			KBR(b)-II & $\hat{C}_{XY|Y} \tilde{\mu}_{Y}$ & $\hat{C}_{XX|Y} \tilde{\mu}_{Y}$ & $(\bar{C}_{XY})^{T} (\bar{C}_{XX}^{2} + \epsilon I)^{-1} \bar{C}_{XX}$ & $\Psi D \big[ (KD)^{2} + m^{2} \epsilon I \big]^{-1} K D \Phi^{T}$ 
		\end{tabular}
	\end{table*}
	
	Bayesian inference often requires computation of the posterior $\mathbb{P}_{Y|X}$ when given the likelihood $\mathbb{P}_{X|Y}$ and the prior $\mathbb{P}_{Y}$. When density evaluations exist, the Bayes' rule provides their relationship as $p_{Y|X}(\cdot|x) = \frac{p_{X|Y}(x|\cdot) p_{Y}(\cdot)}{\int_{\mathcal{Y}} p_{X|Y}(x|y) p_{Y}(y) dy}$.
		
	Nevertheless, several levels of intractability may arise. The first is when both likelihood and prior density evaluations are tractable but the evidence integral $\int_{\mathcal{Y}} p_{X|Y}(x|y) p_{Y}(y) dy$ is intractable, leading to literatures such as \gls{VI} \citep{blei2017variational} and \gls{MCMC} \citep{hastings1970monte}. The next is when only likelihood evaluations are intractable but sampling is possible, leading literatures such as \gls{LFI} and \gls{ABC} \citep{marin2012approximate}. More rarely, only prior evaluations are intractable but available via sampling, leading to literatures in implicit priors. The last is when both the likelihood and prior evaluations are intractable but available via sampling, leading to newer literatures such as implicit generative models.
		
	While there are many approaches that addresses each of these scenarios, the underlying limitation is that Bayes' rule requires density evaluations that are difficult to approximate in high dimensions from samples. Instead, if relationships between the posterior, likelihood, and prior can be captured without using density evaluations, but directly by using samples, this issue could be more naturally sidestepped. Both \glspl{DMO} and \gls{KBR} provide such a nonparametric Bayes' rule. %This is especially useful as implicit distributions generate samples directly.
		
	\setlength{\tabcolsep}{4pt}
	\begin{table}[b]
		\caption{Degenerate case for empirical estimators when $\epsilon = 0$}
		\label{tab:dmo_kbr_comparison_degenerate}
		\centering
		\begin{tabular}{ccc}
			DMO(W) & KBR(a) & KBR(b) \\
			\midrule
			$\tilde{\Psi} A^{T} \big[A A^{T}\big]^{-1} K^{-1} \Phi^{T}$ & $\tilde{\Psi} A^{T} D^{-1} K^{-1} \Phi^{T}$ & $\Psi K^{-1} \Phi^{T}$
		\end{tabular}
	\end{table}
	
	\Cref{tab:dmo_kbr_comparison} compares all four forms of \gls{KBR} \citep{song2013kernel} with \gls{DMO}. This table illustrates the different ways each method estimates the joint and evidence operators from likelihood and prior operators, the type of regularization used for inverting the evidence operator, and the final computational form. For \gls{KBR}, (a) and (b) differ in the joint operator, and I and II differ in the type of regularization used for inverting the evidence operator. Via the Woodbury identity, for \gls{DMO} we also show an alternative computational form \gls{DMO}(W) that better illustrate its contrast with \gls{KBR}(a)-I and \gls{KBR}(b)-I. Note that unlike the four types of \gls{KBR}, \gls{DMO}(W) is the same model as \gls{DMO}, just with a different computational form.
		
	In particular, the diagonal matrix $D := \mathrm{diag}(A \bvec{1})$ arises from the use of third order operators. This can make estimators sensitive to regularizations on inverse operators. This is best seen in the degenerate case of $\epsilon \to 0^{+}$, shown in \cref{tab:dmo_kbr_comparison_degenerate}, where for \gls{KBR}(b) the effect of $\tilde{\bvec{y}}$ vanishes, even though $\epsilon$ does not correspond to regularizations from the prior.
		
	Furthermore, the original computational form of \glspl{DMO} involves the inverse of a positive definite matrix. This however is not true for \gls{KBR}(a) and \gls{KBR}(b) since $KD$ is not symmetric and thus the resulting matrix to be inverted cannot be positive definite. For \gls{KBR}(b), by using $D = D^{\frac{1}{2}} D^{\frac{1}{2}}$ and the Woodbury identity, \gls{KBR}(b)-I and \gls{KBR}(b)-II can be written in forms with symmetric matrix inverses as $\bar{C}_{Y|X} = \Psi D^{\frac{1}{2}} \big[ D^{\frac{1}{2}} K D^{\frac{1}{2}} + m \epsilon I \big]^{-1} D^{\frac{1}{2}} \Phi^{T}$ and  $\bar{C}_{Y|X} = \Psi D^{\frac{1}{2}} \big[ D^{\frac{1}{2}} K D K D^{\frac{1}{2}} + m^{2} \epsilon I \big]^{-1} D^{\frac{1}{2}} K D \Phi^{T}$ respectively. However, it is difficult to interpret this form. 

	Finally, similar to \cref{thm:emp_dmo_special_case}, for the other degenerate case where $m = n$, $\tilde{\bvec{y}} = \bvec{y}$, and $\lambda \to 0^{+}$, all estimators revert to a \gls{CME} $\hat{C}_{Y|X} = \Psi (K + n \epsilon I)^{-1} \Phi^{T}$. 

\newpage
\section{Theorems and Experiment Details for \Cref{sec:experiments}}
\label{sec:experiment_details}
		
	\subsection{Hyperparameter Learning of Deconditional Kernel Mean Embeddings for Likelihood Free Inference}
	
		To learn hyperparameters, we maximize the following objective function which approximate the marginal likelihood of the inference problem.
		
		\begin{theorem}[Approximate Marginal Likelihood for \gls{LFI}]
		\label{thm:aml_lfi}
			Assume $\kappa_{\epsilon}(\bvec{y}, \cdot) \in \mathcal{H}_{k}$ and that $\hat{C}_{\bvec{X} | \bm{\Theta}}$ is a bounded operator for all $n$. Denote $\bm{\kappa}_{\epsilon}(\bvec{y}) = \{\kappa_{\epsilon}(\bvec{y}, \bvec{x}_{i})\}_{i = 1}^{n}$ and $\bvec{1}_{m} = \{1\}_{j = 1}^{m}$, then $\bar{q}(\bvec{y}) := \langle \kappa_{\epsilon}(\bvec{y}, \cdot), \hat{C}_{\bvec{X} | \bm{\Theta}} \tilde{\mu}_{\bm{\Theta}} \rangle_{\mathcal{H}_{k}} = \frac{1}{m} \bm{\kappa}_{\epsilon}^{T} A \bvec{1}_{m}$ is an estimator to the marginal likelihood $p_{\epsilon}(\bvec{y})$ and converge at $O_{p}(m^{-\frac{1}{2}} + (n \lambda)^{-\frac{1}{2}} + \lambda^{\frac{1}{2}})$.
		\end{theorem}	

		\begin{proof}[Proof of \Cref{thm:aml_lfi}]
		
			Consider the absolute difference between $\bar{q}(\bvec{y})$ and $p_{\epsilon}(\bvec{y})$,
			\begin{equation}
			\begin{aligned}
				| \bar{q}(\bvec{y}) - p_{\epsilon}(\bvec{y}) | &\leq | \bar{q}(\bvec{y}) - q(\bvec{y}) | + | q(\bvec{y}) - p_{\epsilon}(\bvec{y}) |.
			\end{aligned}
			\end{equation}
			where $q(\bvec{y}) := \langle \kappa_{\epsilon}(\bvec{y}, \cdot), \hat{C}_{\bvec{X} | \bm{\Theta}} \mu_{\bm{\Theta}} \rangle_{\mathcal{H}_{k}} = \mathbb{E}[\langle \kappa_{\epsilon}(\bvec{y}, \cdot), \hat{C}_{\bvec{X} | \bm{\Theta}} \ell(\bm{\Theta}, \cdot) \rangle_{\mathcal{H}_{k}}]$. The first term is
			\begin{equation}
			\begin{aligned}
				| \bar{q}(\bvec{y}) - q(\bvec{y}) | &= | \langle \kappa_{\epsilon}(\bvec{y}, \cdot), \hat{C}_{\bvec{X} | \bm{\Theta}} (\tilde{\mu}_{\bm{\Theta}} - \mu_{\bm{\Theta}}) \rangle_{\mathcal{H}_{k}}| = | \langle (\hat{C}_{\bvec{X} | \bm{\Theta}})^{T} \kappa_{\epsilon}(\bvec{y}, \cdot), (\tilde{\mu}_{\bm{\Theta}} - \mu_{\bm{\Theta}}) \rangle_{\mathcal{H}_{\ell}}| \\
				&\leq \| (\hat{C}_{\bvec{X} | \bm{\Theta}})^{T} \kappa_{\epsilon}(\bvec{y}, \cdot) \|_{\mathcal{H}_{\ell}} \| (\tilde{\mu}_{\bm{\Theta}} - \mu_{\bm{\Theta}}) \|_{\mathcal{H}_{\ell}} \\
				&\leq c \| (\tilde{\mu}_{\bm{\Theta}} - \mu_{\bm{\Theta}}) \|_{\mathcal{H}_{\ell}}.
			\end{aligned}
			\end{equation}
			for some constant $c$ since $\hat{C}_{\bvec{X} | \bm{\Theta}}$ is a bounded operator for all $n$. Hence, $| \bar{q}(\bvec{y}) - q(\bvec{y}) |$ decays at $O(m^{-\frac{1}{2}})$.
			
			For the second term, we have $p_{\epsilon}(\bvec{y}) = \mathbb{E}[p_{\epsilon}(\bvec{y} | \bm{\Theta})] = \mathbb{E}[\langle \kappa_{\epsilon}(\bvec{y}, \cdot), \mu_{\bvec{X} | \bm{\Theta} = \bm{\Theta}} \rangle_{\mathcal{H}_{k}}] = \mathbb{E}[\langle \kappa_{\epsilon}(\bvec{y}, \cdot), C_{\bvec{X} | \bm{\Theta}} \ell(\bm{\Theta}, \cdot) \rangle_{\mathcal{H}_{k}}]$, similar to $q(\bvec{y}) = \mathbb{E}[\langle \kappa_{\epsilon}(\bvec{y}, \cdot), \hat{C}_{\bvec{X} | \bm{\Theta}} \ell(\bm{\Theta}, \cdot) \rangle_{\mathcal{H}_{k}}]$. Since we use bounded kernels, define $\bar{\ell} := \sup_{\bm{\theta}} \| \ell(\bm{\theta}, \cdot) \|_{\mathcal{H}_{\ell}}$ and $\bar{\kappa}_{\epsilon} := \sup_{\bvec{y}} \| \kappa_{\epsilon}(\bvec{y}, \cdot) \|_{\mathcal{H}_{k}}$. The second term becomes 
			\begin{equation}
			\begin{aligned}
				| q(\bvec{y}) - p_{\epsilon}(\bvec{y}) | &= | \mathbb{E}[\langle \kappa_{\epsilon}(\bvec{y}, \cdot), \hat{C}_{\bvec{X} | \bm{\Theta}} \ell(\bm{\Theta}, \cdot) \rangle_{\mathcal{H}_{k}}] - \mathbb{E}[\langle \kappa_{\epsilon}(\bvec{y}, \cdot), C_{\bvec{X} | \bm{\Theta}} \ell(\bm{\Theta}, \cdot) \rangle_{\mathcal{H}_{k}}] | \\
				&\leq \mathbb{E}[| \langle \kappa_{\epsilon}(\bvec{y}, \cdot), \hat{C}_{\bvec{X} | \bm{\Theta}} \ell(\bm{\Theta}, \cdot) \rangle_{\mathcal{H}_{k}}] - \mathbb{E}[\langle \kappa_{\epsilon}(\bvec{y}, \cdot), C_{\bvec{X} | \bm{\Theta}} \ell(\bm{\Theta}, \cdot) \rangle_{\mathcal{H}_{k}} | ] \\
				&= \mathbb{E}[| \langle \kappa_{\epsilon}(\bvec{y}, \cdot), (\hat{C}_{\bvec{X} | \bm{\Theta}} - C_{\bvec{X} | \bm{\Theta}}) \ell(\bm{\Theta}, \cdot) \rangle_{\mathcal{H}_{k}} |] \\
				&\leq \mathbb{E}[ \| \kappa_{\epsilon}(\bvec{y}, \cdot) \|_{\mathcal{H}_{k}} \|(\hat{C}_{\bvec{X} | \bm{\Theta}} - C_{\bvec{X} | \bm{\Theta}}) \ell(\bm{\Theta}, \cdot) \|_{\mathcal{H}_{k}} ] \\
				&= \| \kappa_{\epsilon}(\bvec{y}, \cdot) \|_{\mathcal{H}_{k}} \mathbb{E}[ \|(\hat{C}_{\bvec{X} | \bm{\Theta}} - C_{\bvec{X} | \bm{\Theta}}) \ell(\bm{\Theta}, \cdot) \|_{\mathcal{H}_{k}} ] \\
				&= \bar{\kappa}_{\epsilon} \mathbb{E}[ \| (\hat{C}_{\bvec{X} | \bm{\Theta}} - C_{\bvec{X} | \bm{\Theta}}) \ell(\bm{\Theta}, \cdot) \|_{\mathcal{H}_{k}} ] \\
				&\leq \bar{\kappa}_{\epsilon} \mathbb{E}[ \| \hat{C}_{\bvec{X} | \bm{\Theta}} - C_{\bvec{X} | \bm{\Theta}} \|_{HS} \| \ell(\bm{\Theta}, \cdot) \|_{\mathcal{H}_{\ell}} ] \\
				&= \bar{\kappa}_{\epsilon} \mathbb{E}[ \| \hat{C}_{\bvec{X} | \bm{\Theta}} - C_{\bvec{X} | \bm{\Theta}} \|_{HS} \sqrt{\ell(\bm{\Theta}, \bm{\Theta})} ] \\
				&= \bar{\kappa}_{\epsilon} \mathbb{E}[ \sqrt{\ell(\bm{\Theta}, \bm{\Theta})} ]  \| \hat{C}_{\bvec{X} | \bm{\Theta}} - C_{\bvec{X} | \bm{\Theta}} \|_{HS} \\
				&\leq \bar{\kappa}_{\epsilon} \mathbb{E}[ \bar{\ell} ]  \| \hat{C}_{\bvec{X} | \bm{\Theta}} - C_{\bvec{X} | \bm{\Theta}} \|_{HS} \\
				&= \bar{\kappa}_{\epsilon} \bar{\ell}  \| \hat{C}_{\bvec{X} | \bm{\Theta}} - C_{\bvec{X} | \bm{\Theta}} \|_{HS}.
			\end{aligned}
			\end{equation}
			Hence, in the worst case $| q(\bvec{y}) - p_{\epsilon}(\bvec{y}) |$ decays at the rate $\| \hat{C}_{\bvec{X} | \bm{\Theta}} - C_{\bvec{X} | \bm{\Theta}} \|_{HS}$ decays, which is $O_{p}((n \lambda)^{-\frac{1}{2}} + \lambda^{\frac{1}{2}})$. Together with the first term, we have the claimed convergence rate.
			
			Finally, the empirical form is obtained from substituting the empirical forms for the likelihood \gls{CMO} and prior embedding, $\bar{q}(\bvec{y}) := \langle \kappa_{\epsilon}(\bvec{y}, \cdot), \hat{C}_{\bvec{X} | \bm{\Theta}} \tilde{\mu}_{\bm{\Theta}} \rangle_{\mathcal{H}_{k}} = \langle \kappa_{\epsilon}(\bvec{y}, \cdot), (\Phi (L + n \lambda I)^{-1} \Psi^{T}) (\frac{1}{m} \tilde{\Psi} \bvec{1}_{m}) \rangle_{\mathcal{H}_{k}} = \frac{1}{m} \bm{\kappa}_{\epsilon}^{T} A \bvec{1}_{m}$.
		\end{proof}
			 
	To satisfy $\kappa_{\epsilon}(\bvec{y}, \cdot) \in \mathcal{H}_{k}$, we use $\kappa_{\epsilon}(\bvec{y}, \bvec{x}) = p_{\epsilon}(\bvec{y} | \bvec{x})= \mathcal{N}(\bvec{y} ; \bvec{x}, \epsilon^{2} I)$ and Gaussian kernel for $k$ with length scale $\epsilon$, so that $\kappa_{\epsilon}$ is just the normalized version of the reproducing kernel $k$. 

	Importantly, while the approximate marginal likelihood $\bar{q}(\bvec{y})$ depends on the hyperparameters of the kernels $k$ and $\ell$ and the regularization $\lambda$, it does not depend on $\epsilon$. At first, it seems that this objective cannot help us learn $\epsilon$. Nevertheless, due to points (4) and (5) of \cref{thm:ttblr_ttbkr}, we have that a good proxy for setting $\epsilon$ once $\lambda$ is learned is $\epsilon = \frac{n}{m} \lambda$.
	
	Nevertheless, for simplicity in our experiments we optimize all kernel hyperparameters and keep the regularization hyperparameters fixed, which has already achieved sufficiently accurate results.
	
	\subsection{Exponential-Gamma Experiment}

		The toy exponential-gamma problem is a standard benchmark for likelihood-free inference, where the true posterior $p_{\epsilon}(\bm{\theta} | \bvec{y})$ is known and tractable even for $\epsilon = 0$. We follow the experimental setup of \citet{meeds2014gps}.
		
		Of the kernel based methods that we have benchmarked against, \gls{K-ABC} \citep{nakagome2013kernel}, \gls{K2-ABC} \citep{park2016k2}, \gls{KBR} \citep{fukumizu2013kernel}, and \gls{KELFI} \citep{hsu2019bayesian} are also \gls{LFI} methods based on the \gls{KME} framework. Consequently, they are very suitable for comparisons towards \gls{DME}. For all these methods, we apply kernel herding on their posterior embeddings to get posterior samples, and plot the approximate posterior density in \cref{fig:lfi} (left) using \gls{KDE} on the posterior samples. In contrast, \gls{GPS-ABC} \citep{meeds2014gps} has its own adaptive \gls{MCMC} based sampling algorithm. We set a simulation budget of 200 simulations and run it until either $10000$ posterior samples are generated or the simulation budget is reached. For hyperparameters, we used standard median heuristic for \gls{K-ABC}, \gls{K2-ABC}, and \gls{KBR}. In contrast, \gls{DME} and \gls{KELFI} have their own marginal likelihoods for hyperparameter learning. For both cases, we find global and local optimums of the marginal likelihood for the hyperparameters and show their results, emphasizing that maximizing the marginal likelihood objective produces better inference results. The hyperparameters of the \gls{GP} surrogate itself used in \gls{GPS-ABC} are learned by maximizing the marginal likelihood of the \gls{GPR} \citep{rasmussen2006gaussian}. However, for hyperparameters of \gls{GPS-ABC} that are not part of the surrogate, we select them based on the original paper \citep{meeds2014gps}. We then report its best two results.
		
	\subsection{Lotka-Volterra Experiment}

		The Lotka-Volterra simulator describes the population dynamics of a well known predator-prey system. For most parameters, the simulation produces chaotic behavior. Realistic scenarios with oscillatory behavior appears only for a small set of parameters.  Consequently, inference on the Lotka-Volterra simulator is extremely challenging.
		
		We follow the setup of \citet{papamakarios2016fast} and \citet{tran2017hierarchical}. There are 4 parameters and 9 normalized summary statistics. We place the same uniform prior on the log parameters and use the same ground truth parameters. After performing inference on all four parameters, we similarly show in \cref{fig:lfi} (right) the marginal posterior distribution for $\log{\theta_{1}}$ in the same format as \citet{papamakarios2016fast} and \citet{tran2017hierarchical}.
		
		For \gls{KBR} \citep{fukumizu2013kernel}, \gls{KELFI} \citep{hsu2019bayesian}, and \gls{DME}, we again sample their posterior mean embeddings with kernel herding to get $10000$ posterior samples. Finally, to compute the $95\%$ interval, we compute the empirical $2.5\%$ quantile and $97.5\%$ quantiles on marginal samples of $\log{\theta_{1}}$ from the $10000$ posterior samples. For \gls{MDN} \citep{papamakarios2016fast} and the two \gls{LFVI} methods \citep{tran2017hierarchical}, we report the results from the original source, as well as their results for \gls{REJ-ABC}, \gls{MCMC-ABC}, and \gls{SMC-ABC}.

\end{document}